\documentclass{article}

\usepackage{amsmath, amssymb, amsthm}
\usepackage{hyperref}
\usepackage{url}
\usepackage{natbib}
\usepackage[margin=1in]{geometry}  

\newtheorem{theorem}{Theorem}[section]
\newtheorem{proposition}[theorem]{Proposition}
\newtheorem{lemma}[theorem]{Lemma}
\newtheorem{corollary}[theorem]{Corollary}
\newtheorem{definition}[theorem]{Definition}
\newtheorem{assumption}[theorem]{Assumption}
\newtheorem{remark}[theorem]{Remark}

\usepackage{amsmath,amsfonts,bm}

\def\1{\bm{1}}

\DeclareMathAlphabet{\mathsfit}{\encodingdefault}{\sfdefault}{m}{sl}
\SetMathAlphabet{\mathsfit}{bold}{\encodingdefault}{\sfdefault}{bx}{n}

\newcommand{\Var}{\mathrm{Var}}

\newcommand{\Cov}{\mathrm{Cov}}

\title{Pathway to $O(\sqrt{d})$ Complexity bound under Wasserstein metric of flow-based models}

\makeatletter
\let\save@mathaccent\mathaccent
\newcommand*\if@single[3]{%
  \setbox0\hbox{${\mathaccent"0362{#1}}^H$}%
  \setbox2\hbox{${\mathaccent"0362{\kern0pt#1}}^H$}%
  \ifdim\ht0=\ht2 #3\else #2\fi
  }
\newcommand*\rel@kern[1]{\kern#1\dimexpr\macc@kerna}
\newcommand*\widebar[1]{\@ifnextchar^{{\wide@bar{#1}{0}}}{\wide@bar{#1}{1}}}
\newcommand*\wide@bar[2]{\if@single{#1}{\wide@bar@{#1}{#2}{1}}{\wide@bar@{#1}{#2}{2}}}
\newcommand*\wide@bar@[3]{%
  \begingroup
  \def\mathaccent##1##2{%
    \let\mathaccent\save@mathaccent
    \if#32 \let\macc@nucleus\first@char \fi
    \setbox\z@\hbox{$\macc@style{\macc@nucleus}_{}$}%
    \setbox\tw@\hbox{$\macc@style{\macc@nucleus}{}_{}$}%
    \dimen@\wd\tw@
    \advance\dimen@-\wd\z@
    \divide\dimen@ 3
    \@tempdima\wd\tw@
    \advance\@tempdima-\scriptspace
    \divide\@tempdima 10
    \advance\dimen@-\@tempdima
    \ifdim\dimen@>\z@ \dimen@0pt\fi
    \rel@kern{0.6}\kern-\dimen@
    \if#31
      \overline{\rel@kern{-0.6}\kern\dimen@\macc@nucleus\rel@kern{0.4}\kern\dimen@}%
      \advance\dimen@0.4\dimexpr\macc@kerna
      \let\final@kern#2%
      \ifdim\dimen@<\z@ \let\final@kern1\fi
      \if\final@kern1 \kern-\dimen@\fi
    \else
      \overline{\rel@kern{-0.6}\kern\dimen@#1}%
    \fi
  }%
  \macc@depth\@ne
  \let\math@bgroup\@empty \let\math@egroup\macc@set@skewchar
  \mathsurround\z@ \frozen@everymath{\mathgroup\macc@group\relax}%
  \macc@set@skewchar\relax
  \let\mathaccentV\macc@nested@a
  \if#31
    \macc@nested@a\relax111{#1}%
  \else
    \def\gobble@till@marker##1\endmarker{}%
    \futurelet\first@char\gobble@till@marker#1\endmarker
    \ifcat\noexpand\first@char A\else
      \def\first@char{}%
    \fi
    \macc@nested@a\relax111{\first@char}%
  \fi
  \endgroup
}
\makeatother

\author{
\centerline{
\begin{tabular}{c c}
Xiangjun Meng\thanks{
  Division of Mathematical Sciences, School of Physical and Mathematical Sciences, Nanyang Technological University,
 Singapore.
  \texttt{xiangjun.meng@ntu.edu.sg}}
&
Zhongjian Wang\thanks{
  Division of Mathematical Sciences, School of Physical and Mathematical Sciences, Nanyang Technological University,
 Singapore.
 \texttt{zhongjian.wang@ntu.edu.sg (Corresponding)}}
\end{tabular}
}
}

\date{December 6, 2025}  

\begin{document}

\maketitle

\begin{abstract}
We provide attainable analytical tools to estimate the error of flow-based generative models under the Wasserstein metric and to establish the optimal sampling iteration complexity bound with respect to dimension as $O(\sqrt{d})$.
We show this error can be explicitly controlled by two parts: the Lipschitzness of the push-forward maps of the backward flow which scales independently of the dimension; and a local discretization error scales $O(\sqrt{d})$ in terms of dimension. The former one is related to the existence of Lipschitz changes of variables induced by the (heat) flow. The latter one consists of the regularity of the score function in both spatial and temporal directions.

These assumptions are valid in the flow-based generative model associated with the F\"{o}llmer process and $1$-rectified flow under the Gaussian tail assumption. As a consequence, we show that the sampling iteration complexity grows linearly with the square root of the trace of the covariance operator, which is related to the invariant distribution of the forward process.
\end{abstract}


\section{Introduction}\label{Intro}
The landscape of deep learning has been fundamentally reshaped by the emergence of powerful generative models, including Generative Adversarial Networks (GANs)~\citep{GAN,arjovsky2017wasserstein}, Variational Auto-encoders (VAEs)~\citep{VAE, kingma2019introduction}, and Normalizing Flows~\citep{NFs, wang2023scientific,wan2020vae}, which have achieved remarkable success in a wide range of applications across modalities like images, audio, and text. These models are capable of learning complex data distributions, allowing them to generate high-quality samples~\citep{achiam2023gpt,  song2020score}. 

Diffusion models (DM) are the state-of-the-art generative models, which can be analyzed via the SDE framework~\citep{song2020score}. With the same forward and backward marginal as DM, flow-based models~\citep{chen2023probability,chen2023restoration} are generative models with deterministic flow given initial distribution, offering a strong basis for statistical inference. This unique feature makes them highly effective in applications such as image and audio synthesis, as well as density estimation~\citep{cheng2024convergence}. 

Early works on DMs and flow-based models provide reverse KL guarantees~\citep{chen2023improved,benton2023nearly,conforti2025kl,li2024accelerating}. However,  for structured data, where the target typically lies on a compact sub-manifold~\citep{tenenbaum2000global,bengio2017deep}, the KL divergence between the backward process and the target is ill-defined. Therefore, one may turn to the analysis of flow-based models under the Wasserstein metric, and in this paper, we will consider the $W_2$ distance in Euclidean space, which is well-defined among distributions with finite second-order moments. One of the central difficulties in the analysis under the $W_2$ distance is the accumulation of local error in the Lyapunov-type estimate. This is in sharp contrast with KL-based analysis ~\citep{altschuler2024shifted,zhu2024inclusive,kim2012generalization} which admits the \textbf{Girsanov}’s theorem (for instance, one in \citet{chen2023improved}) showing the constant scaling of the local error. 

In light of this, \textbf{the main contribution of this paper} is to provide analytical tools that study the accumulation error along the sampling flow under the Wasserstein metric and hence ensure the optimal iteration complexity bound $O(\sqrt{d})$. More precisely, we first analyze the potential asymptotic scaling of the truncation error in terms of the temporal variable and the ambient dimension. Then we bound the accumulation of error by the Lipschitz properties of the push-forward maps of the backward flow. As a justification, we illustrate the attainability of the assumptions by showing the optimal complexity bound in F\"{o}llmer flow under the Gaussian tail assumption. Such an assumption applies to both regular and singular targets (when early stopping technique~\citep{lyu2022accelerating} is applied), extendable to infinite-dimensional settings, with further implications for Bayesian inverse problems.
\subsection{Related work}
\underline{Lipschitz changes of variables}
In the field of PDE, the Lipschitzness of transport maps was initiated by~\citet{caffarelli2000monotonicity}, who constructed such maps between log-concave probability measures. Building on this,~\citet{colombo2015lipschitz} developed global Lipschitz maps for compactly supported perturbations of log-concave measures. An alternative approach beyond optimal transport involves diffusion processes. By leveraging the maximum principle for parabolic PDEs, one can show that log-concavity is preserved along the associated diffusion semigroup~\citep{kim2012generalization}.~\citet{mikulincer2023lipschitz} obtained a sharper Lipschitz constant for measures with bounded support and Gaussian mixtures, improving Caffarelli’s result. Based on this,~\citet{dai2023lipschitz} assume a finite third moment and semi-log-convexity to construct a well-posed unit-time F\"{o}llmer flow whose terminal map is Lipschitz and pushes a Gaussian to a target measure in the unit time interval $[0,1]$.~\citet{neeman2022lipschitz} relaxed Colombo’s compact support requirement to boundedness, and~\citet{fathi2024transportation} extended it to Gaussian in $\mathbb{R}^d$ and uniform spherical measures. Furthermore,~\citet{brigati2024heat} obtained the sharpest Lipschitz bound in this setting without controlling the third-order derivative tensor of potential $\nabla^3 W$. For clarity, we summarize the assumptions on target distributions and their Lipschitz constants in Table~\ref{tab:lip_HF}, with details in Appendix~\ref{Reference}. Despite these results shed light on potential minimal assumption for the convergence guarantee of flow-based models, in later context, we will demonstrate that estimation of the time derivative of velocity field $\partial_t V$ is also crucial on the pathway of optimal complexity bounds. \\
\underline{Continuous flow-based generative models} 
Building on score-based~\citep{song2020score} and denoising diffusion models~\citep{gao2024convergence1},~\citet{salimans2022progressive} introduces stable parameterizations and a distillation method to reduce sampling steps while maintaining sample quality. Flow matching (FM)~\citep{lipman2022flow} extends continuous normalizing flows (CNFs)~\citep{chen2018neural} by training a neural ODE-parameterized vector field $v_\theta(x,t)$ to match a target velocity $v(x,t)$ along fixed probability paths, unifying diffusion and non-diffusion models for efficient and stable generation. Rectified flow~\citep{liu2022flow, rout2024semantic} learns neural ODEs that transport distributions along nearly straight paths through iterative rectification processes, yielding deterministic couplings with reduced transport cost and enabling efficient one-step simulation.
In addition, stochastic interpolants~\citep{albergo2023stochastic,albergo2022building} unify flow-based and diffusion-based methods to bridge arbitrary densities, 
\begin{equation}\label{stocha}
  X_t=I(t,X_0,X_1)+\gamma(t) z, \quad t\in[0,1],
\end{equation}
recovering the Schr\"{o}dinger bridge when the interpolant is optimized~\citep{leonard2013survey}. Recently, Flow Map Matching (FMM)~\citep{boffi2025flow} has accelerated sampling by learning the two-time flow map of generative dynamics, thereby alleviating the computational cost associated with continuous models. \citet{geng2025mean} connect one-step generative modeling to multiscale physical simulations via average velocity, achieving leading performance on ImageNet 256$\times$256 without pre-training or distillation.\\
\underline{Convergence bounds} Recent studies control the KL, $W_2$, and TV distances between the generative and target distributions to guarantee convergence and measure training discretization errors.~\citet{albergo2022building} bounded the $W_2^2$ distance by $e^{1+2K}H(\hat{v})$ under the smoothness and Lipschitz assumptions, where $H(\hat{v})$ measures discrete velocity error. ~\citet{albergo2023stochastic} derived KL-based perturbation bounds for CNF estimators, while FMM~\citep{boffi2025flow} improved $W_2$ bounds for pre-trained models via Lagrangian and Eulerian distillation losses controlling the teacher-student Wasserstein gap. The estimation error of the FM has been analyzed for typical data distributions (e.g., manifold-supported) by~\citet{benton2023error}, and a nonparametric $\mathcal{O}(n^{-1/(d+5)})$ convergence rate under early stopping~\citep{lyu2022accelerating} was established by~\citet{gao2024convergence}, where $n$ denotes the sample size. Subsequently,~\citet{cheng2024convergence} showed JKO flows reach $\mathcal{O}(\epsilon^2)$ KL error in $N\lesssim\log(1/\epsilon)$ steps, extending to non-density cases and yielding mixed KL-$W_2$ guarantees.  We summarize recent complexity results for diffusion models and flow-based models (under Wasserstein distance) in Table~\ref{tab:conver-bound}. Detailed theorems appear in Lemmas~\ref{lemma_ref11}-\ref{lemma_ref17}. In this work, we achieve an optimal
dependence of $\mathcal{O}\left(\sqrt{d}\right)$ on the data dimension $d$ without the assumption of log-concaveness of the target.
\begin{table}[ht]
\vspace{-0.4cm}
\caption{Complexity bounds for DM/flow-based models in $d$ dimensions: previous results vs.\ ours.}
\label{tab:conver-bound}
\begin{center}
\begin{tabular}{lccc}
\multicolumn{1}{c}{\bf Target $P_0$}  &\multicolumn{1}{c}{\bf Complexity}  &\multicolumn{1}{c}{\bf Result}
\\ \hline \\
$P_0$ log-concave$^*$&  $ \mathcal{O}\left(\frac{\sqrt{d}}{\epsilon_0}(\log \frac{d}{\epsilon_0})^2\right)$ & \citet{gao2024convergence1} Tab. 1 \\
G-tail Ass.$^*$ & $\mathcal{O}\left(\frac{\sqrt{d}}{\epsilon_0}\log \frac{d}{\epsilon_0^2}\right)$ & \citet{wang2024wasserstein} Cor. 3.5\\
one-side Lip+weakly log-concave$^*$ & $ \mathcal{O}\left(\frac{d^2}{\epsilon_0^2}\right)$& \citet{gentiloni2025beyond} Thm.3.5\\
weakly log-concave$^*$ & $ \mathcal{O}\left(\frac{d}{\epsilon_0^2}\right)$&  \citet{bruno2025wasserstein} Thm.3.12\\
\textbf{G-tail} Ass.~\ref{assump2} & $\mathcal{O}\left(\frac{\sqrt{d}}{\epsilon_0}\right)$ & This work Thm.\ref{thm4.2}
\end{tabular}
\end{center}
\vspace{-0.4cm}
\begin{flushleft}
\footnotesize $^*$ denotes works on diffusion models; $n$ is the sample size.
\end{flushleft}
\end{table}
\vspace{-2.0em}
\subsection{Contributions}
\begin{itemize}
\item We point out that the $W_2$-distance between the generative and target distributions is controlled by the Lipschitzness of the push-forward maps introduced by sampling flow.  By providing concrete bounds on the Lipschitz coefficient, we obtain an explicit estimate of the accumulation error. 
\item  While prior works often rely on smoothness or strict log-concavity, we adopt a general condition in applications-the Gaussian tail Assumption~\ref{assump2} to provide the well-posedness and Lipschitz regularity of F\"{o}llmer flow, with explicit, dimension-free Lipschitz bounds (Corollary~\ref{thm333} and Corollary~\ref{thm3}). 
\item By leveraging the Gaussian tail Assumption~\ref{assump2} to obtain accurate upper bounds on the time derivative of velocity field $|\partial_t V|$ (Theorem~\ref{thm1}), our framework avoids the need for end-point constraints or early stopping~\citep{lyu2022accelerating}, enabling training and sampling throughout the entire interval $t\in [0,1]$. This framework naturally extends the $\mathcal{O}\left(\sqrt{d}\right)$ complexity results of the SDE flow to the deterministic flow, achieving even better complexity than previous approaches~\citep{wang2024wasserstein}. 
\end{itemize}
\section{Flow-based Model}
We begin by introducing a unified formulation of flow-based generative models. This general framework allows the convergence analysis in Section~\ref{Main results} to apply seamlessly to both the F\"{o}llmer flow and more general sampling dynamics. Consider a continuous flow governed by a velocity field $V$ via the ODE\footnote{We used the left arrow $\overleftarrow{\cdot}$ to represent its connections to the backward process in the score based model.}
\begin{equation}\label{ODE}
 \frac{\mathrm{d}\overset{\leftarrow}{X}_t}{dt}= V\big(t,\overset{\leftarrow}{X}_t\big), \quad \overset{\leftarrow}X_{0}=x, \quad t \in [0,1].
\end{equation}
With the $N$ steps discretization in time, $0 = t_0 < t_1 < \cdots < t_N = 1$, the ODE \eqref{ODE} in each sub-interval $[t_n, t_{n+1}]$, can be interpreted as a local transport map,
\begin{equation}\label{T_n}
T_n(\overset{\leftarrow}{X}_{t_{n}}) = \overset{\leftarrow}{X}_{t_{n+1}}.
\end{equation}
The overall flow-based model $\overset{\leftarrow}{X}_1(x)$ is then obtained by the composition of transport maps
\begin{equation*}
\overset{\leftarrow}{X}_1(x) = (T_{N-1} \circ T_{N-2} \circ \cdots \circ T_0)(x).
\end{equation*}
An approximation of $\overset{\leftarrow}{X}_1(x)$ can be interpreted as approximation of $\{T_n\}_{n=0}^{N-1}$ by $\{\widetilde{T_n}\}_{n=0}^{N-1}$. 
To quantify the error of the approximation, we denote the marginal distribution of the actual state $\overset{\leftarrow}{X}_{t_{n+1}}$ by $\overset{\leftarrow}P_{t_{n+1}}$, and $\overset{\leftarrow}{Q}_{t_{n+1}}$ of the approximated state. Correspondingly we have,
\begin{equation}\label{push}
 \overset{\leftarrow}P_{t_{n+1}}=({T}{_n})_\#(\overset{\leftarrow}{P}_{t_{n}}),  \quad \overset{\leftarrow}{Q}_{t_{n+1}}=({\widetilde{T}}_n)_\#(\overset{\leftarrow}{Q}_{t_{n}}).
\end{equation}
\paragraph{F\"{o}llmer flow}For any $\varepsilon \in (0, 1)$, we consider a diffusion process $(\overset{\rightarrow}{X}_t)_{t \in [0, 1 - \varepsilon]}$  that gradually transforms the target distribution $\nu$ into a Gaussian $\mathcal{N}(0,C)$ over time by the following It\^o SDE
\begin{equation}\label{diffusion}
d\overset{\rightarrow}{X}_t = -\frac{1}{1-t}\overset{\rightarrow}{X}_t dt + \sqrt{\frac{2C}{1-t}}dW_t, \quad \overset{\rightarrow}{X}_0 \sim \nu, \quad t \in [0, 1 - \varepsilon], 
\end{equation}
where $W_t$ is a standard Brownian motion, $C$ is a symmetric, positive-definite covariance matrix. 
The transition probability distribution from $\overset{\rightarrow}{X}_0$ to $\overset{\rightarrow}{X}_t$ is given by 
\begin{equation}\label{forward}
\overset{\rightarrow}{X}_t |\overset{\rightarrow}{X}_0 = x_0 \sim \mathcal{N}\big((1 - t)x_0, t(2 - t)C\big).
\end{equation}
The marginal distribution flow $(\widebar{p}_t)_{t \in [0, 1 - \varepsilon]}$ of the forward diffusion process satisfies the Fokker-Planck-Kolmogorov (FPK) equation in an Eulerian framework
\begin{equation}\label{FP}
\partial_t \widebar{p}_t = \nabla \cdot \left(\widebar{p}_t\cdot \frac{1}{1-t}[x + C\nabla \log \widebar{p}_t (x)]\right) \quad \text{on } [0, 1 - \varepsilon] \times \mathbb{R}^d, \quad \widebar{p}_0 = \nu.
\end{equation}
Then F\"{o}llmer flow is formally defined as the backward process of such a forward diffusion~\eqref{diffusion}, while preserving the same marginal distributions in \eqref{FP}.
\begin{definition}[F\"{o}llmer flow in formal sense] \label{def1}
A F\"{o}llmer flow $(\overleftarrow{X_t})_{t \in [0,1]}$ solves the IVP
\begin{align}
\begin{cases}
\frac{\mathrm{d}\overset{\leftarrow}{X}_t}{\mathrm{d}t} = V(t,\overset{\leftarrow}{X}_t), \quad \overset{\leftarrow}{X}_0 \sim \gamma_C, \quad t \in [0,1],\label{follmer}\\
V(t,x) := \frac{1}{t}\left[x + S(t, x)\right], \quad \forall t \in (0,1]; \qquad V(0,x) := \sqrt{C} \mathbb{E}_{\nu}[X],
\end{cases}
\end{align}
$S(t,x):=C\nabla \log p_{t}(x)$ is the score function with probability density $p_{t}= \widebar{p}_{1-t}$ in forward FKP equation~\eqref{FP}. We call $V(t, x)$ a F\"{o}llmer velocity field.
\end{definition}
Following \eqref{push}, we define $\overset{\rightarrow}{P}_{t_n}$ as the marginal distribution of $\overset{\rightarrow}{X}_{t_n}$ in the forward diffusion process. Given the initial distribution $\overset{\rightarrow}{P}_{0} = P_{\rm data}$, then for all $t \in [0,1]$,
$\overset{\leftarrow}{P}_{t_n} = \overset{\rightarrow}{P}_{1-t{_n}}$.

In practice, the velocity field $V(1-t, x) = \frac{1}{1-t}\left[x + C\nabla \log \widebar{p}_t(x)\right]$ is not available since no closed form expression of $\widebar{p}_t$ is known. To this end, one approximates $V$ by a neural network $\widetilde{V}$. The network is trained by minimizing an $\mathbb{L}_2$ estimation loss,
\begin{equation}\label{loss}
\mathbb{E}_{\widebar{p}_t(x)}  \left\| \widetilde{V}(1-t, x) - \frac{1}{1-t}\left[x + C\nabla \log \widebar{p}_t(x)\right] \right\|^2.
\end{equation}
For simplicity, we introduce the notation $X_t := (1-t)X_0 + \sqrt{t(2-t)C}\,\mathcal{N}$, which shares the same marginal distribution as $\overset{\rightarrow}{X}_t$ in \eqref{forward}. Then the velocity field ${V}(1-t, x)$ can be expressed as a conditional expectation~\citep{EAJAM-15-4},
\begin{equation*}
V(1-t,X)
:=\frac{1}{1-t}\bigl[X+S(1-t,X)\bigr]
=
\mathbb{E}_{X_0|X_t}\left[
\frac{1}{1-t}X_t
-
\frac{X_t-(1-t)X_0}{(1-t)t(2-t)}
\,\bigg|\, X_t=X
\right].
\end{equation*}
With an appropriate weight of the $t$-variable, the loss in \eqref{loss} becomes an approximation of this conditional expectation via mean-squared prediction error,
\begin{equation*}
\begin{aligned}
&\mathbb{E}_{X_0,\; N\sim\mathcal{N}(0,I_d),\; t}\left[
\lambda(t)
\left\|\widetilde{V}\left(1-t,X_t\right)-\frac{1}{1-t}X_t+
\frac{\sqrt{C}\mathcal{N}}{(1-t)\sqrt{t(2-t)}}
\right\|^{2}
\right].
\end{aligned}
\end{equation*}
After training, with $\widetilde{V}(1-t, x)$, one can generate samples of the target distribution via an Euler-type discretization of the continuous-time process, starting from the Gaussian initialization $\gamma_C$,
\begin{equation}\label{discret-follmer}
\frac{\mathrm{d}\overset{\leftarrow}{Y}_t}{dt}= \widetilde{V}(t_n,\overset{\leftarrow}{Y}_{t_n}), \quad \overset{\leftarrow}{Y}_{t_0}\sim \gamma_C, \quad t \in [t_n,t_{n+1}],\quad n=0,1\ldots,N-1.
\end{equation} 
Note that \eqref{discret-follmer} defines the transport map 
$\widetilde{T}_n$ for the learned F\"{o}llmer flow, governed by the approximate velocity field $\widetilde{V}(t_n, \overset{\leftarrow}{Y}_{t_n})$ over the sub-interval $[t_n, t_{n+1}] \subset [0,1]$. Distribution of generation $\overset{\leftarrow}{Q}_t$ is then defined by \eqref{push}. 
\paragraph{Well-posedness of F\"{o}llmer flow}Under appropriate assumptions on the target distribution $\nu$, one can show the F\"{o}llmer flow being the time-reversal of the forward diffusion process~\eqref{diffusion}. For instance, under third moment (Assumption~\ref{assump1}), semi-log-convexity (Assumption~\ref{Semi-log-convexity}) and the structural assumptions (Assumption~\ref{ass:struct}) on $\nu$, 
\citet{dai2023lipschitz} studied the F\"{o}llmer flow in the case $C=I_d$, where the score function is given by
\begin{equation*}
S(t, x) := \nabla \log \int_{\mathbb{R}^d} (2\pi(1-t^2))^{-\frac{d}{2}} \exp \left( -\frac{|x-ty|^2}{2(1-t^2)} \right) \nu(dy).
\end{equation*}
It can be shown that the velocity field $V$ is Lipschitz continuous in $x$ with a well-defined initial condition $V(0,x)$. By the Cauchy-Lipschitz theory~\citep{ambrosio2014continuity}, one can define a Lagrangian flow $(X_t^*)_{t \in [0, 1]}$ governed by the well-posed ODE system,
\begin{equation*}
dX_t^* = -V(1-t, X_t^*)dt, \quad X_0^* \sim \nu, \quad t \in [0, 1],
\end{equation*}
sharing the same marginal distribution with \eqref{diffusion}.

In this work, we study the F\"{o}llmer flow with correlated Gaussian initial based on the Gaussian tail Assumption~\ref{assump2}, and retain the spatially anisotropic noise assumption ($C \neq I_d$) to allow our theory to generalize to infinite-dimensional settings requiring compactification~\citep{lim2025score}; We refer the reader to Theorem~\ref{thm1} for the regularity of the velocity field and Lemma~\ref{thm2} for the proof of well-posedness.

\paragraph{General Notations}
Let $\gamma_C$ denote the density of $\mathcal{N}(0, C)$. For an $n\times n$ matrix $A$, the operator norm $\|\cdot\|$ is defined as
\begin{equation*}
\|A\| = \sup_{v\neq 0} \frac{|Av|}{|v|} = \text{largest eigenvalue of } \sqrt{A^T A}.
\end{equation*}
For symmetric positive-definite $A$, define the weighted $\ell_2$ norm
\begin{equation*}
|x|_A^2 := (A^{-1/2}x, A^{-1/2}x),
\end{equation*}
which reduces to the standard $\ell_2$ norm $|\cdot|$ when $A=I$. For a vector (matrix)-valued function $f(x)$, 
\begin{equation*}
|f|_\infty = \sup_x |f(x)|,\quad (\|f\|_\infty = \sup_x \|f(x)\|).
\end{equation*}

\section{Main results}\label{Main results}
In this section, we present the main results. 
Our analysis begins with a general flow-based framework (not necessarily restricted to the F\"{o}llmer flow), through which we develop Wasserstein-based analytical tools that yield an optimal iteration complexity bound of $\sqrt{d}$.
We then validate the assumptions and present the complexity results for the F\"{o}llmer flow and $1$-rectified flow under the Gaussian tail assumption.
\subsection{Lipschitz changes of variables implies Wasserstein bound of
flow-based models}
For the sake of compactness, we impose the following assumption on the second-order moment.
\begin{assumption}[Second moment]\label{assump1_seconed}
The data distribution has a bounded second moment, $M_2:=\mathbb{E}_{{p}_0} |x|^2 < \infty$.
We further denote,
\begin{equation*}
    M_0 = \max\{\operatorname{Tr}(C), M_2\},
\end{equation*}
relates to the maximum second-order moment, where $C$ is a symmetric, positive-definite covariance matrix. 
\end{assumption}
We consider a general covariance matrix $C$  to cover both the identity case $C = I_d$ and the correlated case $C \neq I_d$. In the main text, we primarily focus on the former, yielding $\operatorname{Tr}(C)=d$ and thus $M_0=\mathcal{O}(d)$ with dimension $d$. At the same time, we retain $C \neq I_d$ in the derivation to extend our theory to infinite-dimensional settings~\citep{lim2025score}, with the general case further discussed in Appendix~\ref{Bayesian_part} for Bayesian inverse problems.
\\
Next, we make three assumptions, each holding with some dimension-free constants. We regard these assumptions as generally valid, and under them, our convergence result Theorem~\ref{444} can be established. 
\begin{assumption}[Lipschitzness of $T$]\label{assump111}
$\forall n=0,\ldots,N-1$, $\text{Lip}({T}_n)<\infty$,  and  $\prod_{j=0}^n \text{Lip}({T}_j)<\infty$.
\end{assumption}
We will justify the attainability of the Assumption~\ref{assump111} in Corollary~\ref{thm333} by invoking the lipschitz property of the velocity field established in Theorem~\ref{thm1}.
Similar to Assumption~\ref{assump111} which imposes Lipschitz continuity of $T$, we also assume the Lipschitz continuity of $\widetilde{T}$ as stated below.
\begin{assumption}[Lipschitzness of $\widetilde{T}$]\label{assump222}
$\forall n=0,\ldots,N-1$, $\text{Lip}(\widetilde{T}_n)<\infty$,  and  $\prod_{j=0}^n \text{Lip}(\widetilde{T}_j)<\infty$.
\end{assumption}
We will verify Assumption~\ref{assump222} in Corollary~\ref{thm3} by leveraging the lipschitz property of the learned velocity field stipulated in
Assumption~\ref{assump4}.
The final assumption concerns the local discretization error between $T$ and $\widetilde{T}$  at each time step $h$,  as described below.
\begin{assumption}[Accuracy of approximation]\label{assump333}
 There exists constants $\widebar{K},\widebar{K_1}, \widebar{K_2},\epsilon$, such that
\begin{equation*}
\sqrt{\mathbb{E}_{x\sim\overset{\leftarrow}P_{t_{n}}} |\widetilde{T}_n(x) - {T}_n(x)|^2} \leq h\left(\Big(\widebar{K}\sqrt{M_0}+\frac{\widebar{K_1}}{\sqrt{1-t_n^2}}+\widebar{K_2}\Big)h+\epsilon\right),
\end{equation*}
with time step size $h=t_{n+1}-t_n$.
\end{assumption}
This scaling follows since
\begin{equation*}
\widetilde{T}_n(x)- {T}_n(x)= h\big(V(x)-\widetilde{V}(x)\big)+\mathcal{O}(h^2),
\end{equation*}
as verified in the F\"{o}llmer case Theorem~\ref{thm4.2}. 
The term $\mathcal{O}(h)$ reflects the $\epsilon$-accuracy of the learned velocity $\widetilde{V}(x)$ (Assumption~\ref{assump3}), while the term $\mathcal{O}(h^2)$ stems from the Taylor expansion of
$T_n(x)$ over $[t_n, t_{n+1}]$ and depends on its regularity, possibly also on ambient dimension $d$ and time $t$.

Next, we outline the core proof strategy of this work. The key step is to demonstrate the Lipschitz continuity of both the original and discretized flows, which is critical for guaranteeing the convergence of flow-based generative models.
\begin{theorem}\label{444}
Assume that the target distribution satisfies Assumption~\ref{assump1_seconed} and follows Lipschitzness Assumption~\ref{assump111}, \ref{assump222}, and approximation error Assumption~\ref{assump333}. With constant step size $h$, the Wasserstein-2 distance between the target distribution $\overset{\rightarrow}P_{0}=\overset{\leftarrow}P_{1}$ and the generation $\overset{\leftarrow}Q_{1}$ is bounded as, 
\begin{equation}\label{eqn:statement of Lip theorem}
\begin{aligned}  
 \mathcal{W}_2(\overset{\leftarrow}P_{1}, \overset{\leftarrow}Q_{1})
 \leq& \left(\prod_{j=0}^{N-1}\text{Lip}(\widetilde{T}_j)\right)\mathcal{W}_2(\overset{\leftarrow}P_{0}, \overset{\leftarrow}Q_{0})\\
 &+h\sum_{k=0}^{N-2}\prod_{j=k}^{N-2}\text{Lip}(\widetilde{T}_j)\left(\Big(\widebar{K}\sqrt{M_0}+\frac{\widebar{K_1}}{\sqrt{1-t_j^2}}+\widebar{K_2}\Big)h+\epsilon\right).
\end{aligned}
\end{equation}
\end{theorem} 
Proof see Appendix~\ref{AppenA.1}.

This result shows that the first term in the bound scales the initial discrepancy $\mathcal{W}_2(\overset{\leftarrow}P_{0}, \overset{\leftarrow}Q_{0})$ by the product of Lipschitz constants $\left(\prod_{j=0}^{N-1}\text{Lip}(\widetilde{T}_j)\right)$, and the second term $\left(\Big(\widebar{K}\sqrt{M_0}+\frac{\widebar{K_1}}{\sqrt{1-t_j^2}}+\widebar{K_2}\Big)h+\epsilon\right)$, captures accumulated discretization errors (Assumption~\ref{assump333})  and a local discretization error scales $\mathcal{O}(\sqrt{M_0})$,  yielding the $\mathcal{O}(\sqrt{d})$ dependence in the isotropic case $C = I_d$.
Similar results \eqref{simi1} and \eqref{simi2} are listed in Proposition~\ref{lemma_ref11} and Proposition~\ref{lemma_ref14}, while the precise scaling of the second term remains unspecified. To be noted, in the limit of $h\to0$, $ h\sum_{k=0}^{N-2}\frac{1}{\sqrt{1-t_k^2}}\to \frac{\pi}{2}$.

Notably, Theorem~\ref{444} is of general validity: it applies to all flow-based models and their discrete counterparts satisfying the relevant assumptions, and is not limited to the F\"{o}llmer case.
\subsection{Analyses of F\"{o}llmer flow under Gaussian tail assumption}
In this section, we focus on the F\"{o}llmer flow and derive the main convergence result based on Theorem~\ref{444} through Lipschitz changes of variables, which plays a central role in our analysis.
\begin{assumption}[Third moment]\label{assump1}
The data distribution has a bounded third moment, i.e. $\mathbb{E}_{{p}_0} |x|^3 < \infty$.
\end{assumption}
We note that the third-moment assumption~\ref{assump1} is only required to ensure well-posedness of F\"{o}llmer flow at $t=0$ in the proof of Lemma~\ref{thm2} (see Appendix \ref{AppenC}). For our complexity bound, the second moment Assumption~\ref{assump1_seconed} is sufficient.

Our analysis is based on the following key assumption that the tail distribution of the target is similar to a Gaussian distribution with covariance matrix $A$. 
\begin{assumption}[\textbf{G-tail}]\label{assump2}
The density of target distribution $\widebar{p}_0\in C^2(\mathbb{R}^d)$ and has the following tail decomposition:
\begin{equation*}
\widebar{p}_0(x) = \exp\left(-\frac{|x|_A^2}{2}\right)\exp(h(x)), 
\end{equation*}
where there are independent of dimension constants such that,
\begin{itemize}
    \item[(i)] $A$ is a symmetric, positive-definite matrix which can be simultaneously diagonalized with $C$, and  
    \begin{equation*}
    \|A\|< \infty,\quad \|C\|< \infty,\quad
    \|AC^{-1}\| < \infty, \quad \|CA^{-1}\| < \infty. 
    \end{equation*}
    \item[(ii)] the remainder term $ h $ follows  
    \begin{equation*}
    |\sqrt{C}\nabla h|_\infty < \infty, \quad \|C\nabla^2h\|_\infty < \infty. 
    \end{equation*}
\end{itemize}
\end{assumption}
The Gaussian tail Assumption~\ref{assump2} generalizes the log-concavity condition in~\citet{ding2023sampling,gao2024convergence} to heavier-than-sub-Gaussian tails while ensuring sufficient decay for well-posedness and convergence. Although stronger than the weak semi-log-concavity assumption of~\citet{chaintron2025propagation,bruno2025wasserstein}, it yields sharper guarantees: weak semi-log-concavity implies $O(d)$ sampling complexity, whereas the Gaussian tail assumption achieves $O(\sqrt{d})$ scaling in a non-log-concave setting and also accommodates realistic distributions such as early stopping, see \eqref{mani}.

The following theorem bounds the Lipschitz constant and the time derivative of the F\"{o}llmer velocity field in \eqref{follmer} under the Gaussian tail Assumption~\ref{assump2}, supporting the Lipschitz changes of variables in Corollary~\ref{thm333} and convergence rate in Theorem~\ref{thm4.2}.
\begin{theorem}[Regularity of the velocity field]\label{thm1}
The Gaussian tail Assumption~\ref{assump2} implies 
the F\"{o}llmer velocity field $V(t,\cdot)$ has the following regularity properties:
\begin{equation}\label{regular}
\begin{aligned}
    |V(t,x)|&\leq K_0+K_2t|x|,\quad  &\forall t \in [0,1],\\
    \|\nabla V(t,\cdot)\|_\infty&\leq (K_1+K_2)t, &\quad  \forall t \in [0,1],\\
    |\partial_tV(t,x)|&\leq K_5|x|+ \frac{K_6}{\sqrt{1-t^2}}+ K_7,&\quad  \forall t \in [0,1),
\end{aligned} 
\end{equation}
where the coefficients are dimension-free constants, given explicitly in Table \ref{tab:coe} of Appendix~\ref{AppenA}.
\end{theorem}
To handle the blow-up of $|\partial_tV(t,x)|$ near $t = 0, 1$,~\citet{ding2023sampling} restrict $t$ to $[\delta, 1 - \delta]$. In particular,~\citet{gao2024convergence} shows Lipschitz continuity of $V$ in $t$ over $[0,1-\delta_0]$ with constant scaling as $\mathcal{O}(\delta_0^{-2})$. In contrast, under our Gaussian tail assumption, the control over the second derivative of the tail allows us to bound $|\partial_t V(t,x)|$
using techniques such as the Brascamp-Lieb inequality~\citep{brascamp1976extensions}. This analysis reveals that the term $\frac{1}{\sqrt{1-t^2}}$ is integrable on $[0,1]$, thus posing no obstacle to convergence, allowing training and sampling over the full interval $t \in [0,1]$. More importantly, this approach, to our knowledge, is the first to yield the improved $\mathcal{O}(\sqrt{d})$ complexity bound, as formally stated in Corollary~\ref{corollary1}.

Detailed proof of Theorem~\ref{thm1} is provided in Appendix~\ref{AppenA}.

\begin{remark}
Motivated by the averaged-velocity construction in MeanFlow~\citep{geng2025mean}, 
we introduce an analogous notion for the F\"{o}llmer flow and define the \emph{averaged F\"{o}llmer velocity} as
\[
\widebar{V}(x,r,t)
:= \frac{1}{t-r} \int_r^{t} V(\tau,x)\, d\tau .
\]
Under the regularity condition~\eqref{regular} satisfied by the F\"{o}llmer velocity field, 
a direct calculation gives the uniform bound
\[
\bigl|\widebar{V}(x,r,t)\bigr|
\;\le\;
K_0 \;+\; \frac{t+r}{2}\, K_2\, |x|,
\]
demonstrating that the averaged F\"{o}llmer velocity preserves the same linear growth property as the original velocity field.
\end{remark}

Under the preceding assumptions and analysis, we establish the well-posedness of the F\"{o}llmer model $(\overset{\leftarrow}{X}_t)_{t \in [0,1]}$ in the following lemma. 
\begin{lemma}[Well-posedness]\label{thm2}
Suppose that the third moment Assumption~\ref{assump1} and the Gaussian tail Assumption~\ref{assump2} hold. 
 Then the F\"{o}llmer velocity field is well-defined at the $t=0$, in the sense that
\begin{equation}\label{V_0}
 V(0,x) :=\lim_{t \rightarrow 0} V(t, x) = \lim_{t \rightarrow 0} \frac{x + S(t, x)}{t} = \sqrt{C}\mathbb{E}_{\widebar{p}_0}[X].
\end{equation}
Consequently, the F\"{o}llmer flow $(\overset{\leftarrow}{X}_t)_{t \in [0,1]}$  is a unique solution to IVP \eqref{follmer}. Moreover, the push-forward measure satisfies
\[
\gamma_C \circ (\overset{\leftarrow}{X}_1)^{-1} = \widebar{p}_0.
\]
\end{lemma}
Proof see Appendix~\ref{AppenC}.
 Under Assumption~\ref{assump2}, we now establish the Lipschitz property of the continuous flow $(\overset{\leftarrow}{X}_t)_{t \in [0,1]}$.
\begin{corollary}[Lipschitzness of continuous flow]\label{thm333}
If $\widebar{p}_0$ follows the Gaussian tail Assumption~\ref{assump2}, then the F\"{o}llmer flow $(\overset{\leftarrow}{X}_t)_{t \in [0,1]}$ is Lipschitz with a dimension-free constant, more precisely,
\begin{equation}\label{Lip_con}
\text{Lip}(\overset{\leftarrow}{X}_1(x)) \leq \|\nabla \overset{\leftarrow}{X}_1(x)\|_{op} \leq \exp \left( \frac{K_1+K_2}{2}\right).
\end{equation}
\end{corollary}
Proof see Appendix~\ref{AppenC.1}.
Bound like \eqref{Lip_con} can also be achieved in \citet{caffarelli2000monotonicity, colombo2015lipschitz, kim2012generalization, mikulincer2023lipschitz,brigati2024heat} under various assumptions, as detailed in Appendix~\ref{Reference}. In general, the constants involved are dimension-free.

To analyze the stability and convergence of the discrete flow, we first assume the following bound on the velocity field approximation error at the discretization points. 
\begin{assumption}[Accuracy of the learned velocity field]\label{assump3}
For each time discretization point $t_n$, the accuracy of learned velocity $\widetilde{V}(t_n,x)$ approximates the true velocity field $V(t_n,x)$ with uniformly bounded error in expectation:
\begin{equation*}
\mathbb{E}_{\overset{\rightarrow}{P}_{1-t_n}} |V(t_n,x) - \widetilde{V}(t_n,x)|^2 \leq \epsilon^2.
\end{equation*}
\end{assumption}
Next, we assume that the learned velocity field inherits the regularity of the continuous flow under the Gaussian tail Assumption~\ref{assump2}.
\begin{assumption}[Regularity of the learned velocity field]\label{assump4}
Assume the learned velocity field $\widetilde{V}(t, x)$  follows
    \begin{equation*}
        \| \nabla \widetilde{V}(t_n, \cdot) \|_{\infty} \leq (K_1+K_ 2+K_8) t_n
    \end{equation*}
for some positive constant $K_8$. 
\end{assumption}
Although the bound may not be small in general, the Assumption~\ref{assump4} is essential for our theoretical analysis and remains reasonable. The learned velocity field $\widetilde{V}(t_n,x)$ is trained to approximate the true velocity field $V(t_n,x)$ in Assumption~\ref{assump3}, which satisfies the required regularity (see Theorem~\ref{thm1}); Moreover, neural networks generally inherits the smoothness and controlled growth induced by the architecture and training process, which prevents uncontrolled behavior in practice. Assumption~\ref{assump4} can further be relaxed in the temporal $t$ direction to require only that the total discrete-time sum of the score gradient is bounded; see Remark~\ref{relax} in Appendix~\ref{re}. \\
We subsequently establish the Lipschitz property of the discrete flow $(\overset{\leftarrow}{Y}_t)_{t\in[0,1]}$ under Assumption~\ref{assump4}.
\begin{corollary}[Lipschitzness of discrete flow]\label{thm3}
The regularity of learned velocity field Assumption~\ref{assump4} implies the Lipschitz property of the learned flow $(\overset{\leftarrow}{Y}_t)_{t \in [0,1]}$ with a dimension-free constant, such that
\begin{equation*}
\text{Lip}(\overset{\leftarrow}{Y}_1(x)) \leq \|\nabla \overset{\leftarrow}{Y}_1(x)\|_{op} \leq \exp \left( \frac{K_1+K_2+K_8}{2}\right),
\end{equation*}
\end{corollary}
Proof see Appendix~\ref{AppenC.2}.

\subsection{Main Convergence theories}
With the Lipschitz properties of the flow established (see Corollary~\ref{thm333} and Corollary~\ref{thm3}), we next quantify how these bounds propagate through the discrete dynamics. Building on Theorem~\ref{444}, the following theorem provides a convergence result in F\"{o}llmer flow case.
\begin{theorem}\label{thm4.2}
Suppose that the third moment Assumption~\ref{assump1}, the Gaussian tail Assumption~\ref{assump2}, the accuracy and regularity assumptions~\ref{assump3}-~\ref{assump4} on the learned velocity field hold. Using the Euler method for the F\"{o}llmer flow with uniform step size $h = t_{n+1} - t_n \leq 1$ ensures $\sqrt{M_0} $ convergence between the target data distribution and the generated distribution:
\begin{equation}\label{con}
\mathcal{W}_2(\overset{\rightarrow}{P}_0, \overset{\leftarrow}{Q}_{1})  \leq \exp\big( \frac{K_1+K_2+K_8}{2}\big)\Bigg(\sqrt{3}\left(K_5\sqrt{M_0}+ K_9\right)h+2\epsilon\Bigg).
\end{equation}
where $K_1, K_2,\ldots,K_9$ are dimensionless constants defined in Theorem~\ref{thm1} and Assumption~\ref{assump4}, with explicit expressions given in Table~\ref{tab:coe}. Furthermore, with the covariance of base distribution $C=I_d$ in the Assumption~\ref{assump1_seconed}, 
$
\mathcal{W}_2(\overset{\rightarrow}{P}_0, \overset{\leftarrow}{Q}_{1}) =\mathcal{O}(\sqrt{d}\,h+\epsilon)$.
\end{theorem}
Proof see Appendix~\ref{AppenDD.1}.
Note that the first term in Theorem~\ref{444}, stemming from the time-propagating discrepancy of the semigroup maps, vanishes in Theorem~\ref{thm4.2} because the F\"ollmer flow $(\overset{\leftarrow}{X}_t)_{t \in [0,1]}$ is well-posed at $t=0$, giving
$\mathcal{W}_2(\overset{\leftarrow}{P}_{0}, \overset{\leftarrow}{Q}_{0})=0$.
Thus, only the accumulated discretization error remains, corresponding to
the second term in Theorem~\ref{444}. 

\begin{corollary}\label{corollary1}
To reach a distribution $\overset{\leftarrow}{Q}_{1}$ such that $\mathcal{W}_2(\overset{\rightarrow}{P}_0, \overset{\leftarrow}{Q}_{1}) =\mathcal{O}(\varepsilon_0)$  with uniform step size $h = t_{n+1} - t_n \leq 1$ requires at most:
\begin{equation*}
h=\mathcal{O}\left(\frac{\epsilon_0}{\sqrt{M_0}}\right),\quad
N=\frac{1}{h}=\mathcal{O}\left(\frac{\sqrt{M_0}}{\epsilon_0}\right),
\end{equation*}
and Assumption~\ref{assump3} to hold with $
\epsilon=\mathcal{O}(\epsilon_0).$
Furthermore, $N=\mathcal{O}\left(\frac{\sqrt{d}}{\epsilon_0}\right)$
under the Assumption~\ref{assump1_seconed} with $C=I_d$.
\end{corollary}
The complexity bound established in Corollary~\ref{corollary1} grows linearly with the square root of the trace of the forward process’s covariance operator, independent of dimension, and thus extends naturally to infinite-dimensional generative models. An illustrative case is provided in Appendix~\ref{Bayesian_part}, where we consider Bayesian inverse problems in function spaces. Proposition~6 in \citet{gao2025wasserstein} establishes that for the standard Gaussian as target distribution, $\mathcal{O}(\sqrt{d})$ complexity bound is optimal. This indicates that our $\sqrt{d}$ dependence stems from intrinsic Gaussian concentration, making the dimensional scaling fundamental rather than algorithm-induced. Notably, in efforts to obtain complexity bounds under assumptions more general than log-concavity, recent works \citep{bruno2025wasserstein} derived an $\mathcal{O}(d)$ bound using the weakly log-concave assumption \citep{conforti2024weak,conforti2023projected}, while \citep{gentiloni2025beyond} obtained an $\mathcal{O}(d^2)$ bound under the similar assumption. These related works are summarized in Table~\ref{tab:conver-bound}.
  
Since the probabilistic ODE (Prob ODE)~\citep{song2020score,gao2024convergence1} can be viewed as a time-rescaled F\"{o}llmer flow, the result of Corollary~\ref{corollary1} also implies that our method improves the computational complexity of the Prob ODE compared to \citet{wang2024wasserstein}. We will provide a detailed discussion in Appendix~\ref{sec:Prob_ODE}.

We further verified that our method extends to the $1$-rectified flow setting~\citep{liu2022flow, rout2024semantic}.
In particular, it applies to the interpolation paths used in the flow built by the first step rectification over independent coupling prior to the recursive construction, and retains the same $\mathcal{O}(\sqrt{d})$ complexity stated in Corollary~\ref{corollary1}. The proof is deferred to Appendix~\ref{Interpolation}. 

\subsection{Convergence under bounded-support assumption}\label{bound}
Real-world data often lie on low-dimensional manifolds, where the distribution is not absolutely continuous with respect to Lebesgue measure in the ambient dimension, and therefore KL bounds may diverge~\citep{pidstrigach2022score}. This motivates the adoption and study of the manifold assumption~\citep{de2022convergence, EAJAM-15-4}, which, under compactness, entails the following bounded-support assumption. 
\begin{assumption}[Bounded-support assumption]\label{Bounded support}
Suppose distribution $p_0$ has compact support with 
$\mathrm{Diam}(\mathrm{Supp}(p_0)) \leq R$ for some constant $R > 0$.
\end{assumption}
Let $q_\sigma = \exp\left(-\frac{|x|^2}{2\sigma^2}\right) * q_0 $, where $q_0$ satisfies the bounded-support Assumption~\ref{Bounded support}. Consider $g(x) = \log q_\sigma(x) + \frac{|x|^2}{2\sigma^2}$, inspired by similar results in~\citep{de2022convergence,mooney2024global,wang2024wasserstein},  we have
\begin{equation}\label{mani}
|\nabla g|_\infty \leq \frac{R}{\sigma^2}, \quad \|\nabla^2 g\|_\infty \leq \frac{2R^2}{\sigma^4}.
\end{equation}
Set $0=t_0< t_1< \cdots < t_N = 1-\delta$ as the discretization points, where the early stopping~\citep{lyu2022accelerating} coefficient $\delta \ll 1$.
By expressing the distribution of the forward process of F\"{o}llmer flow at stopping time $\delta$ in the form $q_\sigma$,  we obtain the correspondences
\begin{equation*}
\sigma^2\longleftrightarrow 1 - (1-\delta)^2, \quad q_0(x)\longleftrightarrow\frac{1}{1-\delta}\overset{\rightarrow}{P}_{0}(\frac{1}{1-\delta}x).
\end{equation*}
Then by Theorem~\ref{thm1}, we get the following Lipschitz bound of the velocity field under Assumption~\ref{Bounded support}.

\begin{corollary}\label{corollary2}
Suppose that the bounded-support Assumption~\ref{Bounded support} holds. Taking 
$C = I_d$ in \eqref{follmer}, and $\ A = (1 - (1-\delta)^2) I_d$ in Assumption~\ref{assump2}, then for all $ t \in[0,1-\delta]$,
\begin{equation*}
\begin{aligned}
    |V(t,x)|&\leq K_0^{*}+K_2^{*}t|x|,\\
    \|\nabla V(t,\cdot)\|_\infty&\leq (K_1^{*}+K_2^{*})t,\\
    |\partial_tV(t,x)|&\leq K_5^{*}|x|+ \frac{K_6^{*}}{\sqrt{1-t^2}}+ K_7^{*},
\end{aligned} 
\end{equation*}
where coefficients are defined in Table~\ref{tab:coemani} of Appendix~\ref{AppenA}. 
\end{corollary}
The proof parallels the corollary in~\citet{wang2024wasserstein}. 
Using the Lipschitz bound from Corollary~\ref{corollary2}, we obtain a bounded-support-version $W_2$ bound by tracking the constants in  Theorem~\ref{thm4.2}.
\begin{theorem}\label{mainfold}
Suppose that the bounded-support Assumption~\ref{Bounded support} and the accuracy and regularity Assumptions~\ref{assump3}, \ref{assump4} hold. Take $C = I_d$, $\delta \ll 1$,  then we have
\begin{equation*}
   \mathcal{W}_2(\overset{\rightarrow}{P}_\delta, \overset{\leftarrow}{Q}_{1-\delta}) \leq\exp\big(\frac{3R^2}{2\delta^2}+\frac{1}{2\delta}+\frac{K_8}{2}\big)
\Bigg(\sqrt{3}\left(K_5^{*}\sqrt{M_0}+ K_9^{*}\right)h+2\epsilon\Bigg),
\end{equation*}
where $K_5^*$ and $K_9^{*}$ are dimension-free constants, whose explicit forms given in Table~\ref{tab:coemani}, and the constant $K_8$ is defined in Assumption~\ref{assump4}.
\end{theorem}
With the result in Theorem~\ref{mainfold}, we can directly compute the complexity bound under the bounded-support assumption with early stopping technique.
\begin{corollary}\label{corollary3}
With $R$ and $\delta$ fixed, achieving a distribution $\overset{\leftarrow}{Q}_{1-\delta}$ such that $\mathcal{W}_2(\overset{\rightarrow}{P}_\delta, \overset{\leftarrow}{Q}_{1-\delta}) = \mathcal{O}(\epsilon_0)$ requires at most: $N=\mathcal{O}\left(\frac{\sqrt{d}}{\epsilon_0}\right)$, and Assumption~\ref{assump3} to hold with $
\epsilon=\mathcal{O}(\epsilon_0).$
\end{corollary}
Noticing that,
\begin{equation*}
 \mathcal{W}_2(\overset{\rightarrow}{P}_\delta, \overset{\rightarrow}{P}_{0}) \leq \sqrt{\mathbb{E}|\overset{\rightarrow}{X}_\delta- \overset{\rightarrow}{X}_0|^2} \leq \sqrt{2d \delta},
\end{equation*}
the complexity bound can also be derived with respect to $\overset{\rightarrow}{P}_0$. More precisely, we consider the following practical scenario. Now we assume $R^2 = \mathcal{O}(d)$,
then optimizing $\delta$ to achieve $
\mathcal{W}_2(\overset{\rightarrow}{P}_0, \overset{\leftarrow}{Q}_{1-\delta}) = \mathcal{O}(\epsilon_0)$ requires at most logarithmic complexity with $\log N=\mathcal{O}\left(\frac{d^3}{\epsilon_0^4}\right)$.

The conclusion and the discussion of future research directions are provided in Appendix~\ref{conclusion}.

\newpage
\bibliography{iclr2026_conference}

@article{wang2024wasserstein,
  title={Wasserstein Bounds for generative diffusion models with Gaussian tail targets},
  author={Wang, Xixian and Wang, Zhongjian},
  journal={arXiv preprint arXiv:2412.11251},
  year={2024}
}

@article{albergo2023stochastic,
  title={Stochastic interpolants: A unifying framework for flows and diffusions},
  author={Albergo, Michael S and Boffi, Nicholas M and Vanden-Eijnden, Eric},
  journal={arXiv preprint arXiv:2303.08797},
  year={2023}
}

@article{ding2023sampling,
  title={Sampling via {F\"{o}llmer} Flow},
  author={Ding, Zhao and Jiao, Yuling and Lu, Xiliang and Yang, Zhijian and Yuan, Cheng},
  journal={arXiv preprint arXiv:2311.03660},
  year={2023}
}

@article{gao2024convergence,
    title={Convergence of Continuous Normalizing Flows for Learning Probability Distributions},
    author={Gao, Yuan and Huang, Jian and Jiao, Yuling and Zheng, Shurong},
    journal={Machine Learning},
    year={2024}
}

@article{dai2023lipschitz,
  title={Lipschitz transport maps via the {F\"{o}llmer} flow},
  author={Dai, Yin and Gao, Yuan and Huang, Jian and Jiao, Yuling and Kang, Lican and Liu, Jin},
  journal={arXiv preprint arXiv:2309.03490},
  year={2023}
}

@article{geng2025mean,
  title={Mean flows for one-step generative modeling},
  author={Geng, Zhengyang and Deng, Mingyang and Bai, Xingjian and Kolter, J Zico and He, Kaiming},
  journal={arXiv preprint arXiv:2505.13447},
  year={2025}
}

@inproceedings{mikulincer2023lipschitz,
  title={On the Lipschitz properties of transportation along heat flows},
  author={Mikulincer, Dan and Shenfeld, Yair},
  booktitle={Geometric Aspects of Functional Analysis: Israel Seminar (GAFA) 2020-2022},
  pages={269--290},
  year={2023},
  organization={Springer}
}

@article{neeman2022lipschitz,
  title={Lipschitz changes of variables via heat flow},
  author={Neeman, Joe},
  journal={arXiv preprint arXiv:2201.03403},
  year={2022}
}

@inproceedings{song2020score,
title={Score-Based Generative Modeling through Stochastic Differential Equations},
author={Song, Yang and Sohl-Dickstein, Jascha and Kingma, Diederik P and Kumar, Abhishek and Ermon, Stefano and Poole, Ben},
booktitle={International Conference on Learning Representations},
year={2021},
url={https://openreview.net/forum?id=PxTIG12RRHS}
}

@article{caffarelli2000monotonicity,
  title={Monotonicity properties of optimal transportation and the fkg and related inequalities},
  author={Caffarelli, Luis A},
  journal={Communications in Mathematical Physics},
  volume={214},
  number={3},
  pages={547--563},
  year={2000},
  publisher={Springer}
}

@article{kim2012generalization,
  title={A generalization of Caffarelli’s contraction theorem via (reverse) heat flow},
  author={Kim, Young-Heon and Milman, Emanuel},
  journal={Mathematische Annalen},
  volume={354},
  number={3},
  pages={827--862},
  year={2012},
  publisher={Springer}
}

@article{colombo2015lipschitz,
  title={Lipschitz changes of variables between perturbations of log-concave measures},
  author={Colombo, Maria and Figalli, Alessio and Jhaveri, Yash},
  journal={Annali Della Scuola Normale Superiore Di Pisa-Classe Di Scienze},
  volume={17},
  number={4},
  pages={1491--1519},
  year={2017},
}

@article{fathi2024transportation,
  title={Transportation onto log-Lipschitz perturbations},
  author={Fathi, Max and Mikulincer, Dan and Shenfeld, Yair},
  journal={Calculus of Variations and Partial Differential Equations},
  volume={63},
  number={3},
  pages={61},
  year={2024},
  publisher={Springer}
}

@article{brigati2024heat,
  title={Heat flow, log-concavity, and Lipschitz transport maps},
  author={Brigati, Giovanni and Pedrotti, Francesco},
  journal={arXiv preprint arXiv:2404.15205},
  year={2024}
}

@article{tenenbaum2000global,
  title={A global geometric framework for nonlinear dimensionality reduction},
  author={Tenenbaum, Joshua B and Silva, Vin de and Langford, John C},
  journal={science},
  volume={290},
  number={5500},
  pages={2319--2323},
  year={2000},
  publisher={American Association for the Advancement of Science}
}

@book{bengio2017deep,
  title={Deep learning},
  author={Bengio, Yoshua and Goodfellow, Ian and Courville, Aaron and others},
  volume={1},
  year={2017},
  publisher={MIT press Cambridge, MA, USA}
}

@inproceedings{gao2024convergence1,
title={Convergence analysis for general probability flow {ODE}s of diffusion models in wasserstein distances},
author={Gao, Xuefeng and Zhu, Lingjiong},
booktitle={28th International Conference on Artificial Intelligence and Statistics (AISTATS)},
volume={258},
year={2025}
}

@article{cheng2024convergence,
  title={Convergence of flow-based generative models via proximal gradient descent in wasserstein space},
  author={Cheng, Xiuyuan and Lu, Jianfeng and Tan, Yixin and Xie, Yao},
  journal={IEEE Transactions on Information Theory},
  volume={70},
  number={11},
  pages={8087--8106},
  year={2024},
  publisher={IEEE}
}

@inproceedings{benton2023nearly,
title={Nearly \$d\$-Linear Convergence Bounds for Diffusion Models via Stochastic Localization},
author={Joe Benton and Valentin De Bortoli and Arnaud Doucet and George Deligiannidis},
booktitle={The Twelfth International Conference on Learning Representations},
year={2024},
url={https://openreview.net/forum?id=r5njV3BsuD}
}

@inproceedings{chen2023improved,
  title={Improved analysis of score-based generative modeling: User-friendly bounds under minimal smoothness assumptions},
  author={Chen, Hongrui and Lee, Holden and Lu, Jianfeng},
  booktitle={International Conference on Machine Learning},
  pages={4735--4763},
  year={2023},
  organization={PMLR}
}

@article{conforti2025kl,
  title={{KL} convergence guarantees for score diffusion models under minimal data assumptions},
  author={Conforti, Giovanni and Durmus, Alain and Silveri, Marta Gentiloni},
  journal={SIAM Journal on Mathematics of Data Science},
  volume={7},
  number={1},
  pages={86--109},
  year={2025},
  publisher={SIAM}
}

@inproceedings{chen2023restoration,
  title={Restoration-degradation beyond linear diffusions: A non-asymptotic analysis for ddim-type samplers},
  author={Chen, Sitan and Daras, Giannis and Dimakis, Alex},
  booktitle={International Conference on Machine Learning},
  pages={4462--4484},
  year={2023},
  organization={PMLR}
}

@article{chen2023probability,
  title={The probability flow ode is provably fast},
  author={Chen, Sitan and Chewi, Sinho and Lee, Holden and Li, Yuanzhi and Lu, Jianfeng and Salim, Adil},
  journal={Advances in Neural Information Processing Systems},
  volume={36},
  pages={68552--68575},
  year={2023}
}

@article{benton2023error,
  title={Error bounds for flow matching methods},
  author={Benton, Joe and Deligiannidis, George and Doucet, Arnaud},
  journal={arXiv preprint arXiv:2305.16860},
  year={2023}
}

@article{batzolis2021conditional,
  title={Conditional image generation with score-based diffusion models},
  author={Batzolis, Georgios and Stanczuk, Jan and Sch{\"o}nlieb, Carola-Bibiane and Etmann, Christian},
  journal={arXiv preprint arXiv:2111.13606},
  year={2021}
}

@inproceedings{GAN,
	author = {Goodfellow, Ian and Pouget-Abadie, Jean and Mirza, Mehdi and Xu, Bing and Warde-Farley, David and Ozair, Sherjil and Courville, Aaron and Bengio, Yoshua},
	booktitle = {Advances in Neural Information Processing Systems},
	publisher = {Curran Associates, Inc.},
	title = {Generative Adversarial Nets},
	volume = {27},
	year = {2014},
	}

@inproceedings{arjovsky2017wasserstein,
	author = {Arjovsky, Martin and Chintala, Soumith and Bottou, L{\'e}on},
	booktitle = {International conference on machine learning},
	organization = {PMLR},
	pages = {214--223},
	title = {Wasserstein generative adversarial networks},
	year = {2017}}

@inproceedings{VAE,
	author = {Diederik P. Kingma and Max Welling},
	booktitle = {2nd International Conference on Learning Representations},
	title = {Auto-Encoding Variational Bayes},
	year = {2014}}

@article{NFs,
	author = {George Papamakarios and Eric Nalisnick and Danilo Jimenez Rezende and Shakir Mohamed and Balaji Lakshminarayanan},
	journal = {Journal of Machine Learning Research},
	number = {57},
	pages = {1--64},
	title = {Normalizing Flows for Probabilistic Modeling and Inference},
	volume = {22},
	year = {2021}}

@Article{wan2020vae,
author = {Wan , Xiaoliang and Wei , Shuangqing},
title = {VAE-KRnet and Its Applications to Variational Bayes},
journal = {Communications in Computational Physics},
year = {2022},
volume = {31},
number = {4},
pages = {1049--1082}
}

@article{wang2023scientific,
	author = {Wang, Hanchen and Fu, Tianfan and Du, Yuanqi and Gao, Wenhao and Huang, Kexin and Liu, Ziming and Chandak, Payal and Liu, Shengchao and Van Katwyk, Peter and Deac, Andreea and others},
	journal = {Nature},
	number = {7972},
	pages = {47--60},
	publisher = {Nature Publishing Group UK London},
	title = {Scientific discovery in the age of artificial intelligence},
	volume = {620},
	year = {2023}}

@article{kingma2019introduction,
  title={An introduction to variational autoencoders},
  author={Kingma, Diederik P and Welling, Max and others},
  journal={Foundations and Trends{\textregistered} in Machine Learning},
  volume={12},
  number={4},
  pages={307--392},
  year={2019},
  publisher={Now Publishers, Inc.}
}

@article{monge1781memoire,
  title={M{\'e}moire sur la th{\'e}orie des d{\'e}blais et des remblais},
  author={Monge, Gaspard},
  journal={Mem. Math. Phys. Acad. Royale Sci.},
  pages={666--704},
  year={1781}
}

@inproceedings{liu2022flow,
title={Flow Straight and Fast: Learning to Generate and Transfer Data with Rectified Flow},
author={Xingchao Liu and Chengyue Gong and qiang liu},
booktitle={The Eleventh International Conference on Learning Representations },
year={2023},
url={https://openreview.net/forum?id=XVjTT1nw5z}
}

@inproceedings{
albergo2022building,
title={Building Normalizing Flows with Stochastic Interpolants},
author={Michael Samuel Albergo and Eric Vanden-Eijnden},
booktitle={The Eleventh International Conference on Learning Representations },
year={2023},
url={https://openreview.net/forum?id=li7qeBbCR1t}
}

@inproceedings{lipman2022flow,
title={Flow Matching for Generative Modeling},
author={Yaron Lipman and Ricky T. Q. Chen and Heli Ben-Hamu and Maximilian Nickel and Matthew Le},
booktitle={The Eleventh International Conference on Learning Representations },
year={2023},
url={https://openreview.net/forum?id=PqvMRDCJT9t}
}

@article{chen2018neural,
  title={Neural ordinary differential equations},
  author={Chen, Ricky TQ and Rubanova, Yulia and Bettencourt, Jesse and Duvenaud, David K},
  journal={Advances in neural information processing systems},
  volume={31},
  year={2018}
}

@article{leonard2013survey,
author = {L{\'e}onard, Christian},
year = {2013},
title = {A survey of the Schr\"{o}dinger problem and some of its connections with optimal transport},
volume = {34},
journal = {Discrete and Continuous Dynamical Systems},
doi = {10.3934/dcds.2014.34.1533}
}

@article{van2021bayesian,
	author = {van de Schoot, Rens and Depaoli, Sarah and King, Ruth and Kramer, Bianca and M{\"a}rtens, Kaspar and Tadesse, Mahlet G and Vannucci, Marina and Gelman, Andrew and Veen, Duco and Willemsen, Joukje and others},
	journal = {Nature Reviews Methods Primers},
	number = {1},
	pages = {1},
	publisher = {Nature Publishing Group UK London},
	title = {Bayesian statistics and modelling},
	volume = {1},
	year = {2021}}

@article{boffi2025flow,
  title={Flow map matching with stochastic interpolants: A mathematical framework for consistency models},
  author={Boffi, Nicholas Matthew and Albergo, Michael Samuel and Vanden-Eijnden, Eric},
  journal={Transactions on Machine Learning Research},
  year={2025}
}

@article{de2022convergence,
title={Convergence of denoising diffusion models under the manifold hypothesis},
author={De Bortoli, Valentin},
journal={Transactions on Machine Learning Research},
issn={2835-8856},
year={2022},
url={https://openreview.net/forum?id=MhK5aXo3gB}
}

@inproceedings{mooney2024global,
title={Global Well-posedness and Convergence Analysis of Score-based Generative Models via Sharp Lipschitz Estimates},
author={Connor Mooney and Zhongjian Wang and Jack Xin and Yifeng Yu},
booktitle={The Thirteenth International Conference on Learning Representations},
year={2025},
url={https://openreview.net/forum?id=r3cWq6KKbt}
}

@article{pidstrigach2022score,
  title={Score-based generative models detect manifolds},
  author={Pidstrigach, Jakiw},
  journal={Advances in Neural Information Processing Systems},
  volume={35},
  pages={35852--35865},
  year={2022}
}

@book{arzela1895sulle,
  title={Sulle funzioni di linee},
  author={Arzela, Cesare},
  year={1895},
  publisher={Gamberini e Parmeggiani}
}

@incollection{cattiaux2014semi,
  title={Semi log-concave Markov diffusions},
  author={Cattiaux, Patrick and Guillin, Arnaud},
  booktitle={S{\'e}minaire de probabilit{\'e}s XLVI},
  pages={231--292},
  year={2014},
  publisher={Springer}
}

@article{eldan2018regularization,
  title={Regularization under diffusion and anticoncentration of the information content},
  author={Eldan, Ronen and Lee, James R},
  journal={Duke Mathematical Journal},
  volume={167},
  number={5},
   pages={969--993},
  year={2018}
}

@article{bruno2025wasserstein,
title={Wasserstein Convergence of Score-based Generative Models under Semiconvexity and Discontinuous Gradients},
author={Stefano Bruno and Sotirios Sabanis},
journal={Transactions on Machine Learning Research},
issn={2835-8856},
year={2025},
url={https://openreview.net/forum?id=vS9iVRB7XF},
note={}
}

@article{achiam2023gpt,
  title={Gpt-4 technical report},
  author={Achiam, Josh and Adler, Steven and Agarwal, Sandhini and Ahmad, Lama and Akkaya, Ilge and Aleman, Florencia Leoni and Almeida, Diogo and Altenschmidt, Janko and Altman, Sam and Anadkat, Shyamal and others},
  journal={arXiv preprint arXiv:2303.08774},
  year={2023}
}

@inproceedings{li2024accelerating,
title={Accelerating Convergence of Score-Based Diffusion Models, Provably},
author={Li, Gen and Huang, Yu and Efimov, Timofey and Wei, Yuting and Chi, Yuejie and Chen, Yuxin},
booktitle={Forty-first International Conference on Machine Learning},
year={2024},
url={https://openreview.net/forum?id=KB6slOUQP9}
}

@article{altschuler2024shifted,
  title={Shifted composition \uppercase\expandafter{\romannumeral 3}: Local error framework for {KL} divergence},
  author={Altschuler, Jason M and Chewi, Sinho},
  journal={arXiv preprint arXiv:2412.17997},
  year={2024}
}

@inproceedings{zhu2024inclusive,
title={Inclusive {KL} Minimization: A Wasserstein-Fisher-Rao Gradient Flow Perspective},
author={Jia-Jie Zhu},
booktitle={Frontiers in Probabilistic Inference: Learning meets Sampling},
year={2025},
url={https://openreview.net/forum?id=clSHHymeIU}
}

@inproceedings{salimans2022progressive,
title={Progressive Distillation for Fast Sampling of Diffusion Models},
author={Tim Salimans and Jonathan Ho},
booktitle={International Conference on Learning Representations},
year={2022},
url={https://openreview.net/forum?id=TIdIXIpzhoI}
}

@article{ambrosio2014continuity,
  title={Continuity equations and {ODE} flows with non-smooth velocity},
  author={Ambrosio, Luigi and Crippa, Gianluca},
  journal={Proceedings of the Royal Society of Edinburgh Section A: Mathematics},
  volume={144},
  number={6},
  pages={1191--1244},
  year={2014},
  publisher={Royal Society of Edinburgh Scotland Foundation}
}

@Article{EAJAM-15-4,
author = {Yubin, Lu and Zhongjian, Wang and Guillaume, Bal},
title = {Mathematical Analysis of Singularities in the Diffusion Model Under the Submanifold Assumption},
journal = {East Asian Journal on Applied Mathematics},
year = {2025},
volume = {15},
number = {4},
pages = {669--700}
}

@article{brascamp1976extensions,
  title={On extensions of the Brunn-Minkowski and Pr{\'e}kopa-Leindler theorems, including inequalities for log concave functions, and with an application to the diffusion equation},
  author={Brascamp, Herm Jan and Lieb, Elliott H},
  journal={Journal of functional analysis},
  volume={22},
  number={4},
  pages={366--389},
  year={1976},
}

@article{gao2025wasserstein,
  title={Wasserstein convergence guarantees for a general class of score-based generative models},
  author={Gao, Xuefeng and Nguyen, Hoang M and Zhu, Lingjiong},
  journal={Journal of machine learning research},
  volume={26},
  number={43},
  pages={1--54},
  year={2025}
}

@article{gentiloni2025beyond,
  title={Beyond log-concavity and score regularity: Improved convergence bounds for score-based generative models in W2-distance},
  author={Gentiloni-Silveri, Marta and Ocello, Antonio},
  journal={arXiv preprint arXiv:2501.02298},
  year={2025}
}

@article{conforti2024weak,
  title={Weak semiconvexity estimates for Schr{\"o}dinger potentials and logarithmic Sobolev inequality for Schr{\"o}dinger bridges},
  author={Conforti, Giovanni},
  journal={Probability Theory and Related Fields},
  volume={189},
  number={3},
  pages={1045--1071},
  year={2024},
  publisher={Springer}
}

@article{conforti2023projected,
  title={Projected Langevin dynamics and a gradient flow for entropic optimal transport},
  author={Conforti, Giovanni and Lacker, Daniel and Pal, Soumik},
  journal={Journal of the European Mathematical Society},
  year={2025}
}

@article{lim2025score,
  title={Score-based diffusion models in function space},
  author={Lim, Jae Hyun and Kovachki, Nikola B and Baptista, Ricardo and Beckham, Christopher and Azizzadenesheli, Kamyar and Kossaifi, Jean and Voleti, Vikram and Song, Jiaming and Kreis, Karsten and Kautz, Jan and others},
  journal={Journal of Machine Learning Research},
  volume={26},
  number={158},
  pages={1--62},
  year={2025}
}

@article{chaintron2025propagation,
  title={Propagation of weak log-concavity along generalised heat flows via Hamilton-Jacobi equations},
  author={Chaintron, Louis-Pierre and Conforti, Giovanni and Eichinger, Katharina},
  journal={arXiv preprint arXiv:2508.07931},
  year={2025}
}

@article{lyu2022accelerating,
  title={Accelerating diffusion models via early stop of the diffusion process},
  author={Lyu, Zhaoyang and Xu, Xudong and Yang, Ceyuan and Lin, Dahua and Dai, Bo},
  journal={arXiv preprint arXiv:2205.12524},
  year={2022}
}

@article{rout2024semantic,
  title={Semantic image inversion and editing using rectified stochastic differential equations},
  author={Rout, Litu and Chen, Yujia and Ruiz, Nataniel and Caramanis, Constantine and Shakkottai, Sanjay and Chu, Wen-Sheng},
  journal={arXiv preprint arXiv:2410.10792},
  year={2024}
}
\bibliographystyle{iclr2026_conference}

\newpage
\appendix
\section*{Appendix}
\section{Statements of Referenced Theorems}\label{Reference} 
For a more comprehensive discussion and comparison, we provide the following related results from the literature.
\begin{proposition}[\citeauthor{caffarelli2000monotonicity}]\label{lemma_ref0}
Let $\mu = \exp(-Q(x))dx$ and $\nu = \exp(-(Q(x) + V(x)))dx$ denote two Borel probability measures on Euclidean space $(\mathbb{R}^n, |\cdot|)$, where $Q$ denotes a quadratic function, i.e.
\begin{equation*}
Q(x) = \langle Ax, x \rangle + \langle b, x \rangle + c,
\end{equation*}
with $A$ positive-definite, and $V$ is a convex function. Then the Brenier optimal-transport map $T = T_{\mathrm{opt}}$ pushing forward $\mu$ onto $\nu$ is a contraction:
\begin{equation*}
\forall x, y \in \mathbb{R}^n, \quad |T(x) - T(y)| \leq |x - y|.
\end{equation*}
\end{proposition}
\begin{proposition}[\citeauthor{kim2012generalization} Thm.1.1]\label{lemma_ref1}
Let $\mu = \exp(-U(x))dx$ and $\nu = \exp(-(U(x) + V(x)))dx$ denote two Borel probability measures on Euclidean space $(\mathbb{R}^n, |\cdot|)$. Assume that $U \in C_{\textup{loc}}^{3,\alpha}(\mathbb{R}^n)$ ($\alpha > 0$) is a convex function of the form:
\begin{equation*}
U(x) = Q(\textup{Proj}_{E_0}x) + \sum_{i=1}^k \rho_i (|\textup{Proj}_{E_i}x|), \quad \forall i = 1,\dots,k,\  \rho_i''' \leq 0 \ \text{on}\ \mathbb{R}_+,
\end{equation*}
where $Q : E_0 \to \mathbb{R}$ is a quadratic function, i.e.
\begin{equation*}
Q(x) = \langle Ax,x\rangle +\langle b,x\rangle+c.
\end{equation*}
And that $V : \mathbb{R}^n \to \mathbb{R}$ is convex and satisfies our symmetry assumptions~\ref{sym}. Then there exists a map $T : \mathbb{R}^n \to \mathbb{R}^n$ pushing forward $\mu$ onto $\nu$ and satisfying our symmetry assumptions, which is a contraction.
\end{proposition}
\begin{definition}[\citeauthor{kim2012generalization} symmetry assumptions]\label{sym}
We will say that a function $ F : \mathbb{R}^n \to \mathbb{R} $ satisfies our symmetry assumptions if it is invariant under the action of the subgroup $ O(E_1, \ldots, E_k) := 1 \times O(E_1) \times \ldots \times O(E_k) $ of the orthogonal group $ O(n) $, or equivalently, if:
\begin{equation*}
\exists \Phi : \mathbb{R}^{\dim E_0 + k} \to \mathbb{R} \text{ so that } F(x) = \Phi(\operatorname{Proj}_{E_0}x, |\operatorname{Proj}_{E_1}x|, \ldots, |\operatorname{Proj}_{E_k}x|).
\end{equation*}
We will similarly say that a map $ T : \mathbb{R}^n \to \mathbb{R}^n $ satisfies our symmetry assumptions if it commutes with the action of the latter subgroup.
\end{definition}

\begin{proposition}[\citeauthor{kim2012generalization} Thm.1.3]\label{lemma_ref2}
Proposition~\ref{lemma_ref1} is also valid when replacing $T$ with the Brenier optimal transport map $T_{opt}$ pushing forward $\mu$ onto $\nu$.
\end{proposition}

\begin{proposition}[\citeauthor{colombo2015lipschitz} Thm.1.1]\label{lemma_ref3}
Let $ V \in C^{1,1}_{\text{loc}}(\mathbb{R}^n) $ be such that $ e^{-V(x)}  dx \in \mathcal{P}(\mathbb{R}^n) $. Suppose that $ V(0) = \inf_{\mathbb{R}^n} V $ and there exist constants $ 0 < \lambda, \Lambda < \infty $ for which $ \lambda  \mathrm{I_d} \leq D^2V(x) \leq \Lambda  \mathrm{I_d} $ for $\text{a.e. } x \in \mathbb{R}^n$. Moreover, let $ R > 0 $, $ q \in C_c^0(B_R) $, and $ c_q \in \mathbb{R} $ be such that $ e^{-V(x)+c_q-q(x)}  dx \in \mathcal{P}(\mathbb{R}^n) $. Assume that $ -\lambda_q  \mathrm{I_d} \leq D^2 q $ in the sense of distributions for some constant $ \lambda_q \geq 0 $. Then, there exists a constant $ C_1 = C_1(R, \lambda, \Lambda, \lambda_q) > 0 $, independent of $ n $, such that the optimal transport map $ T $ that takes $ e^{-V(x)}  dx $ to $ e^{-V(x)+c_q-q(x)}  dx $ satisfies
\begin{equation*}
\|\nabla T\|_{L^\infty(\mathbb{R}^n)} \leq C_1.
\end{equation*}
\end{proposition}

\begin{proposition}[\citeauthor{neeman2022lipschitz} Thm.1.3]\label{lemma_ref4}
Suppose that $ d\mu = e^{-V(x)}  d\gamma $ is a probability measure on $\mathbb{R}^n$, 
where $ D^2 V \geq -\kappa$ (for $\kappa \geq 1$), 
and $ \sup V - \inf V \leq c $. Then $\mu$ is an $L$-Lipschitz push-forward of $\gamma$, for
\begin{equation*}
L = 2(2\kappa)^{e^c}.
\end{equation*}
\end{proposition}

\begin{proposition}[\citeauthor{mikulincer2023lipschitz} Thm.1]\label{lemma_ref5}
Let $\mu$ be a $\kappa$-log-concave probability measure on $\mathbb{R}^d$, and set $D := diam(supp(\mu))$. Then, for the map $\varphi^{\mathrm{flow}} : \mathbb{R}^d \to \mathbb{R}^d$, which satisfies $\varphi^{\mathrm{flow}}_* \gamma_d = \mu$, the following holds:
\begin{enumerate}
\item If $\kappa > 0$ then
\begin{equation*}
\| \nabla \varphi^{\mathrm{flow}}(x) \|_{\mathrm{op}} \leq \frac{1}{\sqrt{\kappa}},
\end{equation*}
for $\mu$-almost every $x$.
\item If $\kappa D^2 < 1$ then
\begin{equation*}
\| \nabla \varphi^{\mathrm{flow}}(x) \|_{\mathrm{op}} \leq e^{\frac{1 - \kappa D^2}{2}} D,
\end{equation*}
for $\mu$-almost every $x$.
\end{enumerate}
\end{proposition}

\begin{proposition}[\citeauthor{brigati2024heat} Thm.1.4]\label{lemma_ref10}
Let $ \mu = e^{-(W+H)} \in L_{+}^1(\mathbb{R}^d) $ be a probability density on $\mathbb{R}^d$ such that $ W $ is $\kappa$-convex for some $ \kappa> 0 $ and $ H $ is $ L $-Lipschitz for some $ L \geq 0 $.
Then, there exists a map $ T^{\mathrm{flow}}: \mathbb{R}^d \to \mathbb{R}^d $ such that $(T^{\mathrm{flow}})_\# \gamma_d = \mu$
and $ T^{\mathrm{flow}} $ is $ \frac {1}{\sqrt{\kappa}}\exp\left(\frac{L^2}{2\kappa}+2\frac{L}{\sqrt{\kappa}}\right)$-Lipschitz.
\end{proposition}

\begin{corollary}[This work Cor.~\ref{thm333} with $A={(\kappa I_d)}^{-1}$ and $C=I_d$]\label{lemma_ref10_1}
Let $\mu=e^{-\left(W(x)-H(x)\right)}$ be a probability density on $\mathbb{R}^d$ such that $W(x)=\frac{\kappa|x|^2}{2}$ while $H(x)$ being $ L $-Lipschitz and $\|\nabla^2H\|_\infty < L_1$ for some $ L, L_1 \geq 0 $. 
Then, there exists a map $ T^{\mathrm{flow}}: \mathbb{R}^d \to \mathbb{R}^d $ such that $(T^{\mathrm{flow}})_\# \gamma_d = \mu$
and $ T^{\mathrm{flow}} $ is $\frac {1}{\sqrt{\kappa}}\exp\left(\frac{L^2+L_1}{2\kappa}\right)$-Lipschitz.
\end{corollary}
\begin{proof}[Sketch of calculation]
Applying Cor.~\ref{thm333} with $A=(\kappa I_d)^{-1}$ and $C=I_d$, Both 
$A$ and $C$ are diagonalizable, which leads to a more refined Lip upper bound:
\begin{equation*}
\begin{aligned}
&\exp\left(\int_0^1 \frac{\|A\|^2(L^2+L_1)t}{(\|A\|t^2 + \|C\|(1-t^2))^2}\,dt+\int_0^1 \frac{\|A-C\|t}{\|A\|t^2 + \|C\|(1-t^2)}\,dt \right)\\
\leq& \exp\left(\frac{\|A\|(L^2+L_1)}{2\|C\|}+\frac{1}{2}\ln\frac{\|A\|}{\|C\|} \right)\\
\leq& \frac {1}{\sqrt{\kappa}}\exp\left(\frac{L^2+L_1}{2\kappa}\right).
\end{aligned}
\end{equation*}
\end{proof}

\begin{table}[t]
\caption{Lip changes of variables via Heat flow}
\label{tab:lip_HF}
\begin{center}
\begin{tabular}{lccc}
\multicolumn{1}{c}{\bf Target $P_0$}  &\multicolumn{1}{c}{\bf Lip-constant}  &\multicolumn{1}{c}{\bf Result}
\\ \hline \\
log-concave+sym.Ass. &1+sym.Ass.&\citet{kim2012generalization} Prop.${\ref{lemma_ref1}}$\\
          $\kappa$-log-concave+$\mathrm{osc}\leq c$&$2(2\kappa)^{e^c}$&\citet{neeman2022lipschitz} Prop.~{\ref{lemma_ref4}}\\
         $\kappa$-log-concave& $e^\frac{1-\kappa D^2}{2}D$& \citet{mikulincer2023lipschitz} Prop.~{\ref{lemma_ref5}}\\
         L-log-Lipschitz& $\frac{1}{\sqrt{\kappa}}e^{\left(\frac{L^2}{2\kappa}+2\frac{L}{\sqrt{\kappa}}\right)}$   & \citet{brigati2024heat} Prop.~{\ref{lemma_ref10}}\\
         \textbf{G-tail} Ass.~\ref{assump2}         & $\frac {1}{\sqrt{\kappa}}e^{\left(\frac{L^2+L_1}{2\kappa}\right)}$   & This work Cor.~{\ref{lemma_ref10_1}}\\
\end{tabular}
\end{center}
\vspace{-0.3cm}
\begin{flushleft}
\footnotesize where $\kappa, c, D, L, L_1, K_1, K_2$ are dimension-independent constant.
\end{flushleft}
\end{table}

\begin{proposition}[\citeauthor{albergo2022building} Pro.3] \label{lemma_ref11}
Let $\rho_t(x)$ be the exact interpolant density and given a velocity field $\hat{v}_t(x)$, let us define $\hat{\rho}_t(x)$ as the solution of the initial value problem
\begin{equation*}
    \partial_t \hat{\rho}_t + \nabla \cdot (\hat{v}_t \hat{\rho}_t) = 0, \quad \hat{\rho}_{t=0} = \rho_0.
\end{equation*}
Assume that $\hat{v}_t(x)$ is continuously differentiable in $(t,x)$ and Lipschitz in $x$ uniformly on $(t,x) \in [0,1] \times \mathbb{R}^d$ with Lipschitz constant $K$. Then the square of the $W_2$ distance between $\rho_1$ and $\hat{\rho}_1$ is bounded by
\begin{equation}\label{simi1}
    W_2^2(\rho_1, \hat{\rho}_1) \leq e^{1+2K} H(\hat{v})
\end{equation}
where $H(\hat{v})$ is the objective function defined as
\begin{equation*}
H(\hat{v}) = \int_0^1 \int_{\mathbb{R}^d} \left|\hat{v}_t(x) - v_t(x)\right|^2 \rho_t(x) \, dx \, dt.
\end{equation*}
\end{proposition}

\begin{proposition}[\citeauthor{albergo2023stochastic} Thm.2.23]\label{lemma_ref12}
Let $\rho$ denote the solution of the Fokker-Planck equation (2.20) with $\epsilon(t) = \epsilon > 0$. Given two velocity fields $\hat{b}, \hat{s} \in C^0([0,1], (C^1(\mathbb{R}^d))^d)$, define
\begin{equation*} 
\begin{aligned}
\hat{b}_{F}(t,x) = \hat{b}(t,x) + \epsilon\hat{s}(t,x), \quad \hat{v}(t,x) = \hat{b}(t,x) + \gamma(t)\dot\gamma(t)\hat{s}(t,x), 
\end{aligned}
\end{equation*} 
where the function $\gamma$ satisfies the properties listed in Definition~\ref{stocha}. Let $\hat{\rho}$ denote the solution to the Fokker-Planck equation
\begin{equation*} 
\partial_t \hat{\rho} + \nabla \cdot (\hat{b}_{F}\hat{\rho}) = \epsilon \Delta \hat{\rho}, \quad \hat{\rho}(0) = \rho_0.
\end{equation*}
Then,
\begin{equation*} 
\text{KL}(\rho_1 \| \hat{\rho}(1)) \le \frac{1}{2\epsilon} \left( \mathcal{L}_{\hat{b}}[\hat{b}] - \min_{\widetilde{b}} \mathcal{L}_{\hat{b}}[\widetilde{b}] \right) + \frac{\epsilon}{2} \left( \mathcal{L}_{\hat{s}}[\hat{s}] - \min_{\widetilde{s}} \mathcal{L}_{\hat{s}}[\widetilde{s}] \right),
\end{equation*}
where $\mathcal{L}_{\hat{b}}[\hat{b}]$ and $\mathcal{L}_{\hat{s}}[\hat{s}]$ are the objective functions defined in 
\begin{equation*} 
\mathcal{L}_{b}[\hat{b}] = \int_{0}^{1} \mathbb{E} \left( \frac{1}{2} |\hat{b}(t, x_t)|^2 - \left( \partial_t I(t, x_0, x_1) + \dot{\gamma}(t)z \right) \cdot \hat{b}(t, x_t) \right) dt
\end{equation*}
and
\begin{equation*} 
\mathcal{L}_{s}[\hat{s}] = \int_{0}^{1} \mathbb{E} \left( \frac{1}{2} |\hat{s}(t, x_t)|^2 + \gamma^{-1}(t)z\cdot \hat{s}(t, x_t) \right)  dt.
\end{equation*}
And
\begin{equation*} 
\text{KL}(\rho_1 \| \hat{\rho}(1)) \le \frac{1}{2\epsilon} \left( \mathcal{L}_{\hat{v}}[\hat{v}] - \min_{\widetilde{v}} \mathcal{L}_{\hat{v}}[\widetilde{v}] \right) + \frac{\sup_{t \in [0,1]} (\gamma(t)\dot\gamma(t) - \epsilon)^2}{2\epsilon} \left( \mathcal{L}_{\hat{s}}[\hat{s}] - \min_{\widetilde{s}} \mathcal{L}_{\hat{s}}[\widetilde{s}] \right).
\end{equation*}
where $\mathcal{L}_{\hat{v}}[\hat{v}]$ is the objective function defined in 
\begin{equation*} 
\mathcal{L}_{v}[\hat{v}] = \int_{0}^{1} \mathbb{E} \left( \frac{1}{2} |\hat{v}(t, x_t)|^2 - \partial_t I(t, x_0, x_1) \cdot \hat{v}(t, x_t) \right) dt.
\end{equation*}
\end{proposition}

\begin{proposition}[\citeauthor{boffi2025flow} Prop.3.9: Lagrangian error bound]\label{lemma_ref13}
Let $X_{s,t}: \mathbb{R}^d \to \mathbb{R}^d$ denote the flow map for a pre-trained stochastic interpolant or diffusion model, and let $\hat{X}_{s,t}: \mathbb{R}^d \to \mathbb{R}^d$ denote an approximate flow map. Given $x_0 \sim \rho_0$, let $\hat{\rho}_1$ be the density of $\hat{X}_{0,1}(x_0)$ and let $\rho_1$ be the target density of $X_{0,1}(x_0)$. Then,
\begin{equation*} \label{eq:3.15}
W_2^2(\rho_1, \hat{\rho}_1) \le e^{1+2\int_0^1 |C_t| dt} \mathcal{L}_{\text{LMD}}(\hat{X}),
\end{equation*}
where $C_t$ is the Lipschitz constant of the drift term.
\end{proposition}

\begin{proposition}[\citeauthor{boffi2025flow} Prop.3.10: Eulerian error bound]\label{lemma_ref14}
Let $X_{s,t}: \mathbb{R}^d \to \mathbb{R}^d$ denote the flow map for a pre-trained stochastic interpolant or diffusion model, and let $\hat{X}_{s,t}: \mathbb{R}^d \to \mathbb{R}^d$ denote an approximate flow map. Given $x_0 \sim \rho_0$, let $\hat{\rho}_1$ be the density of $\hat{X}_{0,1}(x_0)$ and $\rho_1$ be the target density of $X_{0,1}(x_0)$. Then,
\begin{equation} \label{simi2}
W_2^2(\rho_1, \hat{\rho}_1) \le e^1 \mathcal{L}_{\text{EMD}}(\hat{X}).
\end{equation}
\end{proposition}

\begin{proposition}[\citeauthor{benton2023error} Thm.4]\label{lemma_ref15}
Suppose that the following Assumptions hold,
\begin{itemize}
    \item The true and approximate drifts $v^X(\mathbf{x}, t)$ and $v_\theta(\mathbf{x}, t)$ satisfy $\int_0^1 \mathbb{E} \left[ \left\| v_\theta(X_t, t) - v^X(X_t, t) \right\|^2 \right] dt \le \epsilon^2$.
    \item For each $\mathbf{x} \in \mathbb{R}^d$ and $s \in [0, 1]$ there exist unique flows $(Y_{s,t}^X)_{t \in [s,1]}$ and $(Z_{s,t}^X)_{t \in [s,1]}$ starting in $Y_{s,s}^X = \mathbf{x}$ and $Z_{s,s}^X = \mathbf{x}$ with velocity fields $v_\theta(\mathbf{x}, t)$ and $v^X(\mathbf{x}, t)$ respectively. Moreover, $Y_{s,t}^X$ and $Z_{s,t}^X$ are continuously differentiable in $\mathbf{x}$, $s$ and $t$.
    \item The approximate flow $v_\theta(\mathbf{x}, t)$ is differentiable in both inputs. Also, for each $t \in (0, 1)$ there is a constant $L_t$ such that $v_\theta(\mathbf{x}, t)$ is $L_t$-Lipschitz in $\mathbf{x}$.
    \item For some $\lambda \ge 1$, $\alpha_t X_0 + \beta_t X_1$ is $\lambda$-regular for all $t \in [0,1]$.
\end{itemize}
$\gamma_t$ is a concave function on $[0,1]$ which determines the amount of Gaussian smoothing applied at time $t$, and that $\alpha_0 = \beta_1 = 1$ and $\gamma_0 = \gamma_1 = \gamma_{\min}$. Then, for any $v_\theta \in \mathcal{V}$, if $Y$ is a flow starting in $\hat{\pi}_0$ with velocity field $v_\theta$ and $\hat{\pi}_1$ is the law of $Y_1$,
\begin{equation*}
W_2(\hat{\pi}_1, \pi_1) \leq C^{\lambda^{1/2}}\varepsilon\left(\frac{\gamma_{\max}}{\gamma_{\min}}\right)^{2\lambda} + \sqrt{d}\gamma_{\min}
\end{equation*}
where $C = \exp\left\{R\left(\int_0^1 (|\dot{\alpha}_t|/\gamma_t)\ dt + \int_0^1 (|\dot{\beta}_t|/\gamma_t)\ dt\right)\right\}$ with $\gamma_{\min} = \inf_{t\in[0,1]} \gamma_t$, $\gamma_{\max} = \sup_{t\in[0,1]} \gamma_t$.
\end{proposition}

\begin{proposition}[\citeauthor{cheng2024convergence} Prop.5.4] \label{lemma_ref16}
Suppose $q_N=\widetilde{q}_N=q\in\mathcal{P}_2^\gamma$, and the computed transport maps $T_n$ and $S_n$ satisfy the following assumptions:
\begin{itemize}
    \item The learned transport $T_{n+1}$ is non-degenerate and in $L^2(p_n)$; it is invertible on $\mathbb{R}^d$ and $T_{n+1}^{-1}$ is also non-degenerate. In addition, for some $\varepsilon > 0$,
$\exists \xi_{n+1} \in \partial_{W_2} F_{n+1}(p_{n+1})$ s.t.
$\|\xi_{n+1}\|_{p_{n+1}} \leq \varepsilon$.
\item  For $n = N, \ldots, 1$, the computed reverse transport $S_n$ is non-degenerate, in $L^2(\widetilde{q}_n)$, and satisfies that
\begin{equation*}
    \|T_n \circ S_n - I{_d}\| _{\widetilde{q}_n}\leq \varepsilon_{\text{inv}}.
\end{equation*}
\item There is $K > 0$ s.t. $T_n^{-1}$ is Lipschitz on $\mathbb{R}^d$ with Lipschitz constant $e^{\gamma K}$ for all $n = N,\ldots,1$.
\end{itemize}
Then all $q_n$ and $\widetilde{q}_n$ are in $\mathcal{P}_2^\gamma$ and
\begin{equation*}
W_2(\widetilde{q}_0, q_0) \leq \frac{\varepsilon_{\text{inv}}}{\gamma K} e^{\gamma K(N+1)},
\end{equation*}
where $\varepsilon_{\text{inv}}$ denotes Inversion error and $e^{\gamma K}$ is the Lipschitz constant of inverse transport map.
\end{proposition}

\begin{proposition}[\citeauthor{gao2024convergence} Thm.1.2]\label{lemma_ref17}
Suppose that the target distribution has a bounded support, or is strongly log-concave, or is a mixture of Gaussians. Let $n$ be the sample size and $0 < \underline{t} \ll 1$ satisfying $\underline{t} \approx n^{-1/(d+5)}$. By properly setting the deep ReLU network structure and the forward Euler discretization step sizes, the distribution estimation error of the CNFs learned with linear interpolation and flow matching is evaluated by
\begin{equation*}
\mathbb{E}W_2(\hat{p}_{1-\underline{t}}, p_1) = {\mathcal{O}}(n^{-\frac{1}{d+5}}),
\end{equation*}
where the expectation is taken with respect to all random samples, $\hat{p}_{1-\underline{t}}$ is the law of generated data, $p$ is the law of target data, $W_2(\cdot, \cdot)$ is the Wasserstein-2 distance, and a polylogarithmic prefactor in $n$ is omitted.
\end{proposition}

\begin{definition}[\citeauthor{cattiaux2014semi}]
A probability measure $\mu(dx) = \exp(-U(x))dx$ is $\kappa$-semi-log-concave for some $\kappa \in \mathbb{R}$ if its support $\Omega \subseteq \mathbb{R}^d$ is convex and $U \in C^2(\Omega)$ satisfies
\begin{equation*}
    \nabla^2 U(x) \succeq \kappa I_d, \quad \forall x \in \Omega.
\end{equation*}
\end{definition}

\begin{definition}[\citeauthor{eldan2018regularization}]
A probability measure $\mu(dx) = \exp(-U(x))dx$ is $\beta$-semi-log-convex for some $\beta > 0$ if its support $\Omega \subseteq \mathbb{R}^d$ is convex and $U \in C^2(\Omega)$ satisfies
\begin{equation*}
    \nabla^2 U(x) \preceq \beta I_d, \quad \forall x \in \Omega.
\end{equation*}
\end{definition}

\begin{assumption}[\citeauthor{dai2023lipschitz} Semi-log-convexity]\label{Semi-log-convexity}
$\nu$ is $\beta$-semi-log-convex for some $\beta>0$. 
\end{assumption}
\begin{assumption}[\citeauthor{dai2023lipschitz} Structural condition]\label{ass:struct}
Set $D:=\tfrac{1}{\sqrt{2}}\operatorname{diam}(\operatorname{supp}(\nu))$. One of the following holds:
\begin{enumerate}
\item $\nu$ is $\kappa$-semi-log-concave for some $\kappa>0$ with $D\in(0,\infty]$;
\item $\nu$ is $\kappa$-semi-log-concave for some $\kappa\le 0$ with $D\in(0,\infty)$;
\item $\nu = \mathcal N(0,\sigma^2I_d)\ast \rho$ with a probability $\rho$ supported on a ball of radius $R$ in $R^d$.
\end{enumerate}
\end{assumption}

\section{Proofs}
In this part, we provide the detailed proofs of the theories in this paper.
\subsection{Proof of Theorem~\ref{444}}\label{AppenA.1}
\begin{proof}
Recall Monge's Optimal Transport (OT) problem~\citep{monge1781memoire}, which seeks a map $T$ pushing $\mu$ to $\nu$ that minimizes $\displaystyle \int c(x,T(x)) \, d\mu(x)$ subject to $T_\sharp \mu = \nu$, inducing the $p$-Wasserstein distance
\begin{equation*}
\mathcal{W}_p(\mu,\nu) := \Big( \inf_{\gamma \in \Gamma(\mu,\nu)} \mathbb{E}_{(x,y)\sim \gamma}[\|x-y\|^p] \Big)^{1/p},
\end{equation*}
where $\Gamma(\mu,\nu)$ denotes the set of couplings with marginals $\mu$ and $\nu$.

Applying this to the push-forward measures $\overset{\leftarrow}P_{t_{n+1}}$ and $\overset{\leftarrow}Q_{t_{n+1}}$ (cf. definition~\eqref{push}) at $[t_n,t_{n+1}]$ gives
\begin{equation*}
\begin{aligned}
&\mathcal{W}_2(\overset{\leftarrow}P_{t_{n+1}}, \overset{\leftarrow}Q_{t_{n+1}})\\
 \leq& \sqrt{\mathbb{E}_{(\overset{\leftarrow}{Y}_{t_{n}},\overset{\leftarrow}{X}_{t_{n}})~\sim \Gamma(\overset{\leftarrow}Q_{t_{n}}, \overset{\leftarrow}P_{t_{n}})}|\widetilde{T}_n(\overset{\leftarrow}{Y}_{t_{n}}) -T_n(\overset{\leftarrow}{X}_{t_{n}}) |^2}\\
 \leq& \sqrt{\mathbb{E}_{(\overset{\leftarrow}{Y}_{t_{n}},\overset{\leftarrow}{X}_{t_{n}})~\sim \Gamma(\overset{\leftarrow}Q_{t_{n}}, \overset{\leftarrow}P_{t_{n}})}|\widetilde{T}_{n}(\overset{\leftarrow}{Y}_{t_{n}})-\widetilde{T}_{n}(\overset{\leftarrow}{X}_{t_{n}})|^2}+\sqrt{\mathbb{E}_{\overset{\leftarrow}{X}_{t_{n}}~\sim \overset{\leftarrow}P_{t_{n}}}|\widetilde{T}_{n}(\overset{\leftarrow}{X}_{t_{n}})-T_{n}(\overset{\leftarrow}{X}_{t_{n}})|^2}\\
  \leq& \text{Lip}(\widetilde{T}_n)\sqrt{\mathbb{E}_{(\overset{\leftarrow}{Y}_{t_{n-1}},\overset{\leftarrow}{X}_{t_{n-1}})~\sim \Gamma(\overset{\leftarrow}Q_{t_{n-1}}, \overset{\leftarrow}P_{t_{n-1}})}|\widetilde{T}_{n-1}(\overset{\leftarrow}{Y}_{t_{n-1}})-\widetilde{T}_{n-1}(\overset{\leftarrow}{X}_{t_{n-1}})|^2}\\
  &+h\left(\Big(\widebar{K}\sqrt{M_0}+\frac{\widebar{K_1}}{\sqrt{1-t_n^2}}+\widebar{K_2}\Big)h+\epsilon\right),
\end{aligned}
\end{equation*}
where the last inequality uses Assumption~\ref{assump222} and Assumption~\ref{assump333}.
Then applying the discrete Gr\"{o}nwall inequality yields \eqref{eqn:statement of Lip theorem}.
\end{proof}

\subsection{Proof of Theorem~\ref{thm1}}\label{AppenA}
\begin{proof}
For simplicity, denote
\begin{equation}\label{Gg}
G(t,x,y):=\exp\left( -\frac{\left|K(t)(\sqrt{C})^{-1}x -y\right|^2_{B(t)}}{2}\right),\,g(t,x,y):=\partial_t \left(-\frac{\left|K(t)(\sqrt{C})^{-1}x -y\right|^2_{B(t)}}{2}\right).
\end{equation}
Under the Gaussian tail Assumption~\ref{assump2}, the score function of F\"{o}llmer flow can be calculated by
\begin{equation*}
\begin{aligned}
S(t, x):=C\nabla\log p_t 
=C\nabla \log \int_{\mathbb{R}^d} \Big( 2\pi\, \text{det} \big(B(t)\big) \Big)^{-\frac{d}{2}} G(t,x,y)\cdot \exp\left(-\frac{|x|^2_{\widebar{A}_{t}}}{2}\right)\exp \left( h(\sqrt{C}y) \right)\, \mathrm{d}y,
\end{aligned}\label{S_full}
\end{equation*}
where  $\widebar{A}_{t}=At^2+C(1-t^2)$, $K(t)=(A\widebar{A}_{t}^{-1})t$, $B(t)=(A\widebar{A}_{t}^{-1})(1-t^2)$.

First, we consider the modified score function over the time interval $(0,1]$, 
\begin{equation}\label{tildes}
\begin{aligned}
\widetilde{S}(t,x)&=S(t,x)+C\widebar{A}_{t}^{-1}x\\
&=C\nabla \log \int_{\mathbb{R}^d} \big( 2\pi B(t) \big)^{-\frac{d}{2}}G(t,x,y)\exp \left( h(\sqrt{C}y)\right) \, \mathrm{d}y\\
&=C\frac{\displaystyle\int_{\mathbb{R}^d}\nabla_xG(t,x,y)\exp \left( h(\sqrt{C}y)\right) \mathrm{d}y}{\displaystyle\int_{\mathbb{R}^d}G(t,x,y)\exp \left( h(\sqrt{C}y)\right) \mathrm{d}y}\\
&=-\frac{{K(t)\sqrt{C}}\displaystyle\int_{\mathbb{R}^d}\nabla_yG(t,x,y)\exp \left( h(\sqrt{C}y)\right) \mathrm{d}y}{\displaystyle\int_{\mathbb{R}^d}G(t,x,y)\exp \left( h(\sqrt{C}y)\right) \mathrm{d}y}\\
&=\frac{{K(t)\sqrt{C}}\displaystyle\int_{\mathbb{R}^d}G(t,x,y)\nabla_yh(\sqrt{C}y)\exp \left( h(\sqrt{C}y)\right)\mathrm{d}y}{\displaystyle\int_{\mathbb{R}^d}G(t,x,y)\exp \left( h(\sqrt{C}y)\right)\mathrm{d}y},
\end{aligned}
\end{equation}
Here, the last equal sign is derived from integration by parts.

Since $G(t,x,y)\exp \left( h(\sqrt{C}y)\right) \geq 0$, 
\begin{equation*}
|\widetilde{S}(t,x)|\leq |K(t)\sqrt{C}\nabla h(\sqrt{C}x)|.
\end{equation*}
Let $K=\sup_{0\leq t\leq 1}{|\frac{1}{t}K(t)|}=\sup_{0\leq t\leq 1}{|A \widebar{A}_{t}^{-1}|}\leq\max\{1, \|AC^{-1}\|\}$, we have
\begin{equation*}
|\widetilde{S}(t,\cdot)|\leq K\|C\|^{1/2}|\sqrt{C}\nabla h|_\infty t= K_0t.
\end{equation*}
Taking the derivative twice along that direction and using the same method as above, we get :
\begin{equation*}
\|\nabla\widetilde{S}(t,\cdot)\|_\infty \leq K(t)^2(\|C\nabla ^2 h\|_\infty+|\sqrt{C}\nabla h|^2_\infty)=K_1t^2.
\end{equation*}
where  $K_1:=K^2(\|C\nabla ^2 h\|_\infty+|\sqrt{C}\nabla h|^2_\infty)$. 

Define $K_2:=\sup_{0\leq t\leq1}\|\frac{1}{t^2}(I-C\widebar{A}_{t}^{-1})\|=\sup_{0\leq t\leq1}\|(A-C)(At^2+C(1-t^2))^{-1}\|$, then $\|(I-C\widebar{A}_{t}^{-1})\| \leq K_2t^2$. 

Recall definition of $V(t,x)$ in \eqref{follmer}, we have 
\begin{equation*}\label{lip_bound}
\begin{aligned}
    |V(t,x)|=\left|\frac{x+S(t,\cdot)}{t}\right|=\left|\frac{\widetilde{S}(t,\cdot)+(I-C\widebar{A}_{t}^{-1})x}{t}\right|\leq K_0+K_2t|x|,
\end{aligned} 
\end{equation*}
which implies the velocity field $|V(t,x)|$ remains locally uniformly bounded and grows at most linearly. 

Furthermore,
\begin{equation*}
\begin{aligned}
    \|\nabla V(t,x)\|_\infty=\left\|\nabla\left(\frac{\widetilde{S}(t,x)+(I-C\widebar{A}_{t}^{-1})x}{t}\right)\right\|_\infty\leq (K_1+K_2)t,
\end{aligned}
\end{equation*}

which yields a Lipschitz constant that is uniform over space and independent of the dimension, ensuring uniform equicontinuity.

Next, we give the properties of $V(t,x)$ with respect to time over $(0,1)$.
By taking \eqref{tildes}, we have
\begin{equation}\label{A1+A2}
\begin{aligned}
   \partial_tV(t,x)=&\partial_t\left(\frac{\widetilde{S}(t,\cdot)+(I-C\widebar{A}_{t}^{-1})x}{t}\right)\\    =&\partial_t\left(\frac{K(t)\sqrt{C}\displaystyle\int_{\mathbb{R}^d}G(t,x,y)\nabla_yh(\sqrt{C}y)\exp \left( h(\sqrt{C}y)\right)\mathrm{d}y}{t\displaystyle\int_{\mathbb{R}^d}G(t,x,y)\exp \left( h(\sqrt{C}y)\right)\mathrm{d}y}\right)+\partial_t\left(\frac{(I-C\widebar{A}_{t}^{-1})x}{t}\right)\\
    :=&A_1+A_2
\end{aligned} 
\end{equation}

We begin by calculating the first part,
\begin{equation}\label{A_1}
\begin{aligned}
A_1=&\partial_t\left(\frac{{{(A \widebar{A}_{t}^{-1})}\sqrt{C}}\displaystyle\int_{\mathbb{R}^d}G(t,x,y)\nabla_yh(\sqrt{C}y)\exp \left( h(\sqrt{C}y)\right)\mathrm{d}y}{\displaystyle\int_{\mathbb{R}^d}G(t,x,y)\exp \left( h(\sqrt{C}y)\right)\mathrm{d}y}\right)\\
=&\frac{\partial_t{(A \widebar{A}_{t}^{-1})}\sqrt{C}\displaystyle\int_{\mathbb{R}^d}G(t,x,y)\nabla_yh(\sqrt{C}y)\exp \left( h(\sqrt{C}y) \right)\mathrm{d}y}{\displaystyle\int_{\mathbb{R}^d}G(t,x,y)\exp \left( h(\sqrt{C}y)\right)\mathrm{d}y}\\
&+\frac{A \widebar{A}_{t}^{-1}\sqrt{C}\displaystyle\int_{\mathbb{R}^d} G(t,x,y)g(t,x,y)\exp \left( h(\sqrt{C}y)\right)\nabla_yh(\sqrt{C}y)dy}{\displaystyle\int_{\mathbb{R}^d}G(t,x,y)\exp \left( h(\sqrt{C}y)\right)\mathrm{d}y}\\
&-\frac{A \widebar{A}_{t}^{-1}\sqrt{C} \displaystyle\int_{\mathbb{R}^d} G(t,x,y)\exp \left( h(\sqrt{C}y)\right)\nabla_yh(\sqrt{C}y)dy}{\displaystyle\int_{\mathbb{R}^d}G(t,x,y)\exp \left( h(\sqrt{C}y)\right)\mathrm{d}y}\\
&\quad\cdot\frac{\displaystyle\int_{\mathbb{R}^d} G(t,x,y) g(t,x,y)\exp \left( h(\sqrt{C}y)\right)dy}{\displaystyle\int_{\mathbb{R}^d}G(t,x,y)\exp \left( h(\sqrt{C}y)\right)\mathrm{d}y}\\
\leq& \frac{{\partial_t{(A \widebar{A}_{t}^{-1})}\sqrt{C}}\displaystyle\int_{\mathbb{R}^d}G(t,x,y)\nabla_yh(\sqrt{C}y)\exp \left( h(\sqrt{C}y)\right)\mathrm{d}y}{\displaystyle\int_{\mathbb{R}^d}G(t,x,y)\exp \left( h(\sqrt{C}y)\right)\mathrm{d}y}\\
&+A \widebar{A}_{t}^{-1}\sqrt{C} \,\Cov_{p_t{(y|x)}}\left(g(t,x,y),\nabla_y h(\sqrt{C}y)\right)\\
\leq& \frac{{\partial_t{(A \widebar{A}_{t}^{-1})}\sqrt{C}}\displaystyle\int_{\mathbb{R}^d}G(t,x,y)\nabla_yh(\sqrt{C}y)\exp \left( h(\sqrt{C}y)\right)\mathrm{d}y}{\displaystyle\int_{\mathbb{R}^d}G(t,x,y)\exp \left( h(\sqrt{C}y)\right)\mathrm{d}y}\\
&+A \widebar{A}_{t}^{-1}\sqrt{C} \sqrt{\Var\left( g(t,x,y)\right)}\sqrt{\Var\left(\nabla_y h(\sqrt{C}y)\right)}\\
:= &  I_1+I_2,
\end{aligned} 
\end{equation}
where
\begin{equation*}
p_t{(y|x)=\frac{G(t,x,y)\exp \left( h(\sqrt{C}y)\right)}{\displaystyle\int_{\mathbb{R}^d}G(t,x,y)\exp \left( h(\sqrt{C}y)\right)\mathrm{d}y}}.
\end{equation*}
Define $K_3:=\sup_{0\leq t\leq 1}\|\frac{A\partial_t(\widebar{A}_{t}^{-1})}{t}\|=\sup_{0\leq t\leq 1}\|2A(A-C)((At^2+C(1-t^2))^2)^{-1}\|$, we obtain the first term of $A_1$, 
\begin{equation}\label{I_1}
\begin{aligned}
|I_1|&\leq \left\|2A(A-C)\left(\big(At^2+C(1-t^2)\big)^2\right)^{-1}t\right\||\sqrt{C}||\sqrt{C}\nabla h|_{\infty}\\
&\leq |\sqrt{C}||\sqrt{C}\nabla h|_{\infty}K_3t\\
&\leq K_0K_3K^{-1}.
\end{aligned} 
\end{equation}
For the second term $I_2$ in \eqref{A_1}, the analysis is carried out separately for the following two cases:
\paragraph{Case $\mathrm{I}$ - $\|B(t)\|\leq \frac{1}{2\|C\nabla^2 h\|_{\infty}}$:} Then
     $p_t{(y|x)}$ is a log-concave measure as $-\nabla^2\log p_t{(y|x)}=B(t)^{-1}-C\nabla^2 h\succ 0$. 
Then, by Brascamp-Lieb inequality, we have
\begin{equation}\label{BL}
\begin{aligned}
 \sqrt{\Var\left( g(t,x,y)\right)}\leq \sqrt{\mathbb{E}_{p_t{(y|x)}}\left|\nabla^T g(t,x,y)\Big(-\nabla^2\log p_t{(y|x)}\Big)^{-1} \nabla g(t,x,y)\right|},
 \end{aligned}
\end{equation}
here $g(t,x,y)$ is defined in \eqref{Gg} and its gradient is given by
\begin{equation}\label{g}
\begin{aligned}
\nabla g(t,x,y)= \partial_tK(t)B(t)^{-1}(\sqrt{C})^{-1}x-\partial_tB(t)B(t)^{-2}\left(K(t)(\sqrt{C})^{-1}x -y\right).
\end{aligned}
\end{equation}
Recall that $K(t)=(A\widebar{A}_{t}^{-1})t$, $B(t)=(A\widebar{A}_{t}^{-1})(1-t^2)$.

Substituting \eqref{g} into \eqref{BL} yields
\begin{equation*}
\begin{aligned}
&\text{RHS of \eqref{BL}} \\
=&\Bigg(\int_{\mathbb{R}^d}\Big|\left(\partial_tK(t)B(t)^{-1}(\sqrt{C})^{-1}x-\partial_tB(t)B(t)^{-2}\big(K(t)(\sqrt{C})^{-1}x -y\big)\right)^T\left(B(t)^{-1}-C\nabla^2 h\right)^{-1}\\
&\quad\cdot\left(\partial_tK(t)B(t)^{-1}(\sqrt{C})^{-1}x-\partial_tB(t)B(t)^{-2}\big(K(t)(\sqrt{C})^{-1}x -y\big)\right)p_t(y|x)\Big|\mathrm{d}y\Bigg)^{\frac12}\\
\leq&\Bigg(\int_{\mathbb{R}^d}\Big|\left(\partial_tK(t)B(t)^{-1}C^{-1}x\right)^T\left(I-B(t)C\nabla^2 h\right)^{-1}\partial_tK(t)xp_t(y|x)\Big|\mathrm{d}y\\
&+\int_{\mathbb{R}^d}\Bigg|-\left(2\partial_tK(t)B(t)^{-1}(\sqrt{C})^{-1}\partial_tB(t)x\right)^T\left(I-B(t)C\nabla^2 h\right)^{-1}\left(\big(K(t)(\sqrt{C})^{-1}x-y\big)B(t)^{-1}\right)p_t(y|x)\Bigg|\mathrm{d}y\\
&+\int_{\mathbb{R}^d}\Bigg|\left(K(t)(\sqrt{C})^{-1}x-y\right)^TB(t)^{-2}\big(\partial_tB(t)\big)^2(I-B(t)C\nabla^2 h)^{-1}B(t)^{-1} \left(K(t)(\sqrt{C})^{-1}x-y\right)p_t(y|x)\Bigg|\mathrm{d}y\Bigg)^{\frac12}\\
\end{aligned}
\end{equation*}
Applying the triangle inequality $|a+b+c|^\frac{1}{2}\leq |a|^\frac{1}{2}+|b|^\frac{1}{2}+|c|^\frac{1}{2}$, we obtain:
\begin{equation*}
\begin{aligned}
&\text{RHS of \eqref{BL}}\\
\leq & \Bigg(\int_{\mathbb{R}^d}\Big|\left(\partial_tK(t)B(t)^{-1}C^{-1}x\right)^T\left(I-B(t)C\nabla^2 h\right)^{-1}\partial_tK(t)xp_t(y|x)\Big|\mathrm{d}y\Bigg)^{\frac12}\\
&+\Bigg(\int_{\mathbb{R}^d}\Big|-\left(2\partial_tK(t)B(t)^{-1}(\sqrt{C})^{-1}\partial_tB(t)x\right)^T\left(I-B(t)C\nabla^2 h\right)^{-1}\left(\big(K(t)(\sqrt{C})^{-1}x-y\big)B(t)^{-1}\right)p_t(y|x)\Big|\mathrm{d}y\Bigg)^{\frac12}\\
&+\Bigg(\int_{\mathbb{R}^d}\Big|\left(K(t)(\sqrt{C})^{-1}x-y\right)^TB(t)^{-2}\big(\partial_tB(t)\big)^2(I-B(t)C\nabla^2 h)^{-1}B(t)^{-1} \left(K(t)(\sqrt{C})^{-1}x-y\right)p_t(y|x)\Big|\mathrm{d}y \Bigg)^{\frac12}\\
:=&R_1+R_2+R_3,
\end{aligned}
\end{equation*}
By a straightforward estimation,
\begin{equation*}
\begin{aligned}
R_1\leq |\partial_tK(t)|\|\sqrt{C}^{-1}\|\sqrt{\|B(t)\|^{-1}}\sqrt{(1-\|B(t)\|\|C\nabla^2 h\|)^{-1}}|x|.
\end{aligned}
\end{equation*}
It follows from integration by parts that
\begin{equation*}
\begin{aligned}
 -\int_{\mathbb{R}^d}\left(K(t)(\sqrt{C})^{-1}x-y)B(t)^{-1}\right)p_t(y|x)\mathrm{d}y=-\frac{ \int_{\mathbb{R}^d}\nabla_y G(t,x,y)\exp\left(h(\sqrt{C}y)\right)\mathrm{d}y}{\displaystyle\int_{\mathbb{R}^d}G(t,x,y)\exp \left( h(\sqrt{C}y)\right)\mathrm{d}y}\leq |\sqrt{C}\nabla h|_\infty,
\end{aligned}
\end{equation*}
which immediately implies
\begin{equation*}
\begin{aligned}
R_2&=\Bigg(\frac{\int_{\mathbb{R}^d}\Big|(-2\partial_tK(t)(\sqrt{C})^{-1}xB(t)^{-1}\partial_tB(t)\left(I-B(t)C\nabla^2 h\right)^{-1}\nabla_y G(t,x,y)\exp(h(\sqrt{C}y)\Big|\mathrm{d}y}{\Big|\displaystyle\int_{\mathbb{R}^d}G(t,x,y)\exp \left( h(\sqrt{C}y)\right)\mathrm{d}y\Big|}\Bigg)^\frac{1}{2}\\
&\leq \sqrt{2\|\partial_tK(t)\|\partial_tB(t)\||\nabla h|_\infty}\sqrt{\|B(t)\|^{-1}}\sqrt{(1-\|B(t)\|\|C\nabla^2 h\|)^{-1}}\sqrt{|x|}.
\end{aligned}
\end{equation*}
Similarly, applying integration by parts twice, we obtain
\begin{equation*}
\begin{aligned}
R_3\leq& \|\partial_tB(t)\| \sqrt{\|C(|\nabla h|_\infty^2+\|\nabla^2 h\|_\infty)\|} \sqrt{\|B(t)^{-1}\|}\sqrt{(1-\|B(t)\|\|C\nabla^2h\|_{\infty})^{-1}}\\
&+\|\partial_tB(t)\|\|B(t)^{-1}\|\sqrt{(1-\|B(t)\|\|C\nabla^2h\|_{\infty})^{-1}}.
\end{aligned}
\end{equation*}

Then we have
\begin{equation*}
\begin{aligned}
&\text{RHS of \eqref{BL}} \\
\leq&\sqrt{|\partial_tK(t)x|^2\|C\|^{-1}+2\|\partial_tK(t)\|\|\partial_tB(t)\||\nabla h|_{\infty}|x|+\|\partial_tB(t)\|^2(\|C\nabla^2 h\|_{\infty}+|\sqrt{C}\nabla h|_{\infty}^2)}\\
&\quad\cdot\sqrt{(\|B(t)\|\big(1-\|B(t)\|\|C\nabla^2h\|_{\infty})\big)^{-1}}+
\|\partial_tB(t)\|\sqrt{\big(\|B(t)\|^2(1-\|B(t)\|\|C\nabla^2h\|_{\infty})\big)^{-1}}.
\end{aligned}
\end{equation*}
Combining with
\begin{equation*}
\begin{aligned}
\sqrt{\Var\left(\nabla_y h(\sqrt{C}y)\right)}\leq \|C\nabla^2 h\|_{\infty}\|\|\sqrt{B(t)}\|\|\sqrt{(1-\|B(t)\|\|C\nabla^2h\|_{\infty})^{-1}},
\end{aligned}
\end{equation*}
we derive that
\begin{equation*}
\begin{aligned}
|I_2|\leq& K\|C\nabla^2 h\|_{\infty}(1-\|B(t)\|\|C\nabla^2h\|_{\infty})^{-1}\Big(|\partial_tK(t)x|+\sqrt{2\|\partial_tK(t)\|\| \partial_tB(t)\||C\nabla h|_{\infty}|x|}\\
&+\|\partial_tB(t)\|\sqrt{ \|C\nabla^2 h\|_{\infty}
+|\sqrt{C}\nabla h|_{\infty}^2}\Big)\\
&+ \|\sqrt{B(t)^{-1}}\|K\|\sqrt{C}\|\|\partial_tB(t)\|\|C\nabla^2h\|_{\infty}{(1-\|B(t)\|\|C\nabla^2h\|_{\infty})^{-1}}\\
\leq& 2K\|C\nabla^2 h\|_{\infty}\big(|\partial_tK(t)x|+\sqrt{2 \|\partial_tK(t)\|\|\partial_tB(t)\||\nabla h|_{\infty}|x|}\\
&+\|\partial_tB(t)\|\sqrt{\|C\nabla^2 h\|_{\infty}+|\sqrt{C}\nabla h|_{\infty}^2}+ 2K\|\sqrt{C}\| \|\partial_tB(t)\|\|C\nabla^2h\|_{\infty}\|\sqrt{B(t)^{-1}}\|\\
\leq&  2K_1C_1 |x|+2\sqrt{2}K_1\sqrt{K_0}\sqrt{C_1}\sqrt{C_2}\sqrt{|x|}+ 2K_1^{\frac32}C_2+ 2K_1\|\sqrt{C}\|C_2\sqrt{\|B(t)\|^{-1}}\\
\leq&  3K_1C_1|x|+\frac{2K^{-1}K_1\|\sqrt{C}\|C_2 }{\sqrt{1-t^2}}+C_3,
\end{aligned} 
\end{equation*}
where the second inequality use the fact $(1-\|B(t)\|\|C\nabla^2h\|_{\infty})^{-1}\leq 2$ under Case $\mathrm{I}$, while the last inequality is obtained using Young’s inequality, $C_1, C_2, C_3$ are dimension-free constants defined in table~\ref{tab:coe}.

\paragraph{Case $\mathrm{II}$ - $\|B(t)\|> \frac{1}{2\|C\nabla^2 h\|}$:} According to the definition of variance
\begin{equation*}
\begin{aligned}
\Var(g(t,x,y))=&\int_{\mathbb{R}^d}\big(g(t,x,y)\big)^2p_t{(y|x)}\mathrm{d}y-\left(\int_{\mathbb{R}^d}g(t,x,y)p_t{(y|x)}\mathrm{d}y\right)^2\\
\leq&\int_{\mathbb{R}^d}\big(g(t,x,y)\big)^2p_t{(y|x)}\mathrm{d}y,
\end{aligned}
\end{equation*}
We have the following estimate, analogous to the right-hand side of \eqref{BL}:
 \begin{equation*}
\begin{aligned}
 &\sqrt{\Var\left( g(t,x,y)\right)}\\
     \leq& \frac{\|\sqrt{B(t)^{-1}}\|}{2}\Big(2\|\sqrt{C}^{-1}\||\partial_tK(t)x|+2\sqrt{3}\sqrt{\|\partial_tK(t)\|\| \partial_tB(t)\||\nabla h|_{\infty}|x|}\\
     &+\sqrt{6}\|\partial_tB(t)\|\sqrt{\|C\nabla^2 h\|_{\infty}+|\sqrt{C}\nabla h|_{\infty}^2}\Big)\\
     &+|\partial_tK(t)x|\sqrt{\|\nabla^2 h\|_{\infty}+|\nabla h|_{\infty}^2}+\sqrt{\|\partial_tK(t)\|\|\partial_tB(t)\||\nabla h|_{\infty}(\|C\nabla^2 h\|_{\infty}+|\sqrt{C}\nabla h|_{\infty}^2)|x|}\\
 &+\frac{1}{2}\|\partial_tB(t)\|\sqrt{|\sqrt{C}\nabla h|_{\infty}(\|C\nabla^2 h\|_{\infty}+|\sqrt{C}\nabla h|_{\infty}^2)}+\frac{\sqrt{3}}{2}\|\partial_tB(t)\|\|B(t)^{-1}\|\\
 \leq&\sqrt{\frac{\|\nabla^2 h\|_{\infty}}{2}}\Big(2|\partial_tK(t)x|+2\sqrt{3}\sqrt{\|\partial_tK(t)\|\|\partial_tB(t)\||C\nabla h|_{\infty}|x|}\\
 &+\sqrt{6}\|\sqrt{C}\|\|\partial_tB(t)\|\sqrt{\|C\nabla^2 h\|_{\infty}+|\sqrt{C}\nabla h|_{\infty}^2}\Big)\\
&+|\partial_tK(t)x|\sqrt{\|\nabla^2 h\|_{\infty}+|\nabla h|_{\infty}^2}+\sqrt{\|\partial_tK(t)\|\|\partial_tB(t)\||\nabla h|_{\infty}(\|C\nabla^2 h\|_{\infty}+|\sqrt{C}\nabla h|_{\infty}^2)|x|}\\
&+ \frac{1}{2}\|\partial_tB(t)\|\sqrt{|\sqrt{C}\nabla h|_{\infty}(\|C\nabla^2 h\|_{\infty}+|\sqrt{C}\nabla h|_{\infty}^2)}+\sqrt{3}\|\partial_tB(t)\|\|C\nabla^2 h\|_{\infty}.
\end{aligned}
\end{equation*} 
Together with
\begin{equation*}
\begin{aligned}
\sqrt{\Var\left(\nabla_y h(\sqrt{C}y)\right)}\leq | \sqrt{C}\nabla h|_{\infty},
\end{aligned}
\end{equation*} 
we obtain
\begin{equation*}
\begin{aligned}
|I_2|\leq& K|\sqrt{C}\nabla h|_{\infty}
\Bigg(|\partial_tK(t)x|\big(\sqrt{2\|C\nabla^2 h\|_{\infty}}+\sqrt{(\|C\nabla^2 h\|_{\infty}+|\sqrt{C}\nabla h|_{\infty}^2)}\big)\\
&+\sqrt{6}\sqrt{\|C\nabla^2 h\|_{\infty}}\sqrt{\|\partial_tK(t)\|\|\partial_tB(t)\||C\nabla h|_{\infty}|x|}\\
&+\sqrt{\|\partial_tK(t)\|\|\partial_tB(t)\|\sqrt{C}\nabla h|_{\infty}(\|C\nabla^2 h\|_{\infty}+|\sqrt{C}\nabla h|_{\infty}^2)|x|}\\
&+\sqrt{3}\|\sqrt{C}\|\sqrt{\|C\nabla^2 h\|_{\infty}}\|\|\partial_tB(t)\|\|\sqrt{\|C\nabla^2 h\|_{\infty}+|\sqrt{C}\nabla h|_{\infty}^2}\\
&+\frac12\|\sqrt{C}\|\|\partial_tB(t)\|\sqrt{|\sqrt{C}\nabla h|_{\infty}(\|C\nabla^2 h\|_{\infty}+|\sqrt{C}\nabla h|_{\infty}^2)}+\sqrt{3}\|\sqrt{C}\|\|\partial_tB(t)\|\|C\nabla^2 h\|_{\infty}\Bigg)\\
\leq& (1+\sqrt{2})K_0\sqrt{K_1}C_1|x|+(1+\sqrt{6})K_0^{\frac32}\sqrt{K_1}\sqrt{C_1}\sqrt{C_2}|\sqrt{x}|\\
&+K_0C_2(2\sqrt{3}K_1+\frac12\sqrt{K_0}\sqrt{K_1})\\
\leq& 2(1+\sqrt{2})K_0\sqrt{K_1}C_1|x|+ C_4,
\end{aligned} 
\end{equation*}
where $C_4$ defined in table~\ref{tab:coe} is also constant indepent of dimension. 

Combining the above two cases, we obtain
\begin{equation}\label{I_2}
\begin{aligned}
|I_2|\leq\max\{3K_1C_1,2(1+\sqrt{2})K_0\sqrt{K_1}C_1\}|x|+\frac{2K^{-1}K_1\|\sqrt{C}\|C_2 }{\sqrt{1-t^2}}+ \max\{C_3, C_4\}.
\end{aligned} 
\end{equation}
Using \eqref{A_1}, \eqref{I_1} and \eqref{I_2},  we derive 
\begin{equation*}
\begin{aligned}
|A_1|\leq&|I_1|+|I_2|\\
\leq&K_0K_3K^{-1}+\max\{3K_1C_1,2(1+\sqrt{2})K_0\sqrt{K_1}C_1\}|x|+\frac{2K_1K^{-1}\|\sqrt{C}\|C_2 }{\sqrt{1-t^2}}+ \max\{C_3, C_4\}.
\end{aligned} 
\end{equation*}

The next step is to calculate the absolute value of the second term of \eqref {A1+A2}, i.e.
\begin{equation*}
\begin{aligned}
|A_2|&=\left|\frac{-C\partial_t(\widebar{A}_{t}^{-1})xt+(I-C\widebar{A}_{t}^{-1})x}{t^2}\right|\\
&\leq \|2C(A-C)\big((At^2+C(1-t^2))^2\big)^{-1}\||x|+K_2|x|\\
&\leq (K_4+K_2)|x|.
\end{aligned} 
\end{equation*}
where 
$ K_4:=\sup_{0\leq t\leq 1}\|\frac{C\partial_t(\widebar{A}_{t}^{-1})}{t}\|=\sup_{0\leq t\leq 1}\|2C(A-C)\big((At^2+C(1-t^2))^2\big)^{-1}\|$.

It then follows from \eqref{A1+A2} that
\begin{equation*}\label{proof12}
\begin{aligned}
|\partial_tV(t,x)|\leq |A_1|+|A_2|\leq  K_5|x|+\frac{K_6}{\sqrt{1-t^2}}+ K_7.
\end{aligned}
\end{equation*}
where  $K_5= \max\{3K_1C_1,2(1+\sqrt{2})K_0\sqrt{K_1}C_1\}+K_2+K_4$, $K_6=2K_1K^{-1}\|\sqrt{C}\|C_2$, $K_7=\max\{C_3, C_4\}+K_0K_3{K}^{-1}.$

When $t=0$, by \eqref{V_0} in the subsequently stated well-posedness Lemma~\ref{thm2}, we have
\begin{equation*}\label{V0}
     |V(0,x)|=\left| \sqrt{C}  \mathbb{E}_{\widebar{p}_0}[X] \right|
\lesssim
K \|C\|^{1/2} \left| \sqrt{C} \nabla h \right|_\infty=K_0.
\end{equation*}
Similarly, $\|\nabla V(0,\cdot)\|_\infty$ is also bounded by $K_1+K_2$. Then we conclude that the velocity field $V(t,x)$
satisfied the condition of Theorem~\ref{thm1} for all $t \in [0,1]$. 

For clarity, the coefficients are summarized in Table~\ref{tab:coe}.
\begin{table}[ht]
\caption{Explicit for coefficients in Thm.~\ref{thm1} and Thm.~\ref{thm4.2}.}
\label{tab:coe}
\begin{center}
\begin{tabular}{lcc}
\multicolumn{1}{c}{\bf Coefficient}  &\multicolumn{1}{c}{\bf Explicit expressions}  
\\ \hline \\
$\widebar{A}_{t}$ &  $At^2+C(1-t^2)$\\
$K(t)$ & $(A\widebar{A}_{t}^{-1})t$\\
$B(t)$ &  $(A\widebar{A}_{t}^{-1})(1-t^2)$\\
$K$ &$\sup_{0\leq t\leq 1}{|A \widebar{A}_{t}^{-1}|}\leq\max\{1, \|AC^{-1}\|\}$\\
$K_0$& $K\|C\|^{1/2}|\sqrt{C}\nabla h|_\infty$\\
$K_1$& $K^2(\|C\nabla ^2 h\|_\infty+|\sqrt{C}\nabla h|^2_\infty)$\\
$K_2$& $\sup_{0\leq t\leq 1}\|\frac{1}{t^2}(I-C\widebar{A}_{t}^{-1})\|=\sup_{0\leq t\leq 1}\|(A-C)(At^2+C(1-t^2))^{-1}\|$\\
$K_3$& $\sup_{0\leq t\leq 1}\|\frac{A\partial_t(\widebar{A}_{t}^{-1})}{t}\|=\sup_{0\leq t\leq 1}\|2A(A-C)\big((At^2+C(1-t^2))^2\big)^{-1}\|$\\
$K_4$& $\sup_{0\leq t\leq 1}\|\frac{C\partial_t(\widebar{A}_{t}^{-1})}{t}\|=\sup_{0\leq t\leq 1}\|2C(A-C)\big((At^2+C(1-t^2))^2\big)^{-1}\|$\\
$K_5$& $\max\{3K_1C_1,2(1+\sqrt{2})K_0\sqrt{K_1}C_1\}+K_2+K_4$\\
$K_6$& $2K_1{K}^{-1}\|C\|^{\frac12}C_2$\\
$K_7$& $\max\{C_3, C_4\}+K_0K_3{K}^{-1}$\\
$K_9$& $\frac{1}{4}K_6\pi+K_7$\\
$C_1$& $\max_{t\in [\sqrt{1-\frac{1}{2\|C\|\|\nabla^2 h\|_{\infty}}},1)}\{\partial_tK(t)\}=\max\{\frac{\|A\|(\|(A-C)(1-\frac{1}{2\|C\|\|\nabla^2 h\|_{\infty}})-C\|)}{(\|C\|+\|(A-C)\|(1-\frac{1}{2\|C\|\|\nabla^2 h\|_{\infty}}))^2}, \frac{\|A-2C\|}{\|A\|},\frac{\|A\|}{8\|C\|}\}$\\
$C_2$& $\max_{t\in [\sqrt{1-\frac{1}{2\|C\|\|\nabla^2 h\|_{\infty}}},1)}\{\partial_tB(t)\}=\max\{2,\frac{2\|A\|^2\sqrt{1-\frac{1}{2\|C\|\|\nabla^2 h\|_{\infty}}}}{(\|C\|+\|(A-C)\|(1-\frac{1}{2\|C\|\|\nabla^2 h\|_{\infty}}))^2}, \frac{9A^2}{8\|C\|^2}\sqrt{\frac{\|C\|}{3\|A-C\|}}\}$\\
$C_3$& $2K_1C_2(K_1^{\frac12}+K_0)$\\
$C_4$& $2\sqrt{3}K_0K_1C_2+\frac12K_0^{\frac32}K_1^{\frac12}C_2+\frac{(1+\sqrt{6})^2}{4(1+\sqrt{2})}K_0^2K_1^{\frac12}C_2$\\
\end{tabular}
\end{center}
\end{table}
In corollary~\ref{corollary2}, we take $C = I_d$, $A = (1 - (1-\delta)^2) I_d$, $\widebar{A}_t = (1 - t^2) I_d$. According to \eqref{mani}, we deduce the explicit for coefficients $K_0^{*}, K_1^{*}, K_2^{*}, K_3^{*}, K_4^{*}, K_5^{*}, K_6^{*}, K_7^{*}, K_9^*$ in Table~\ref{tab:coemani}.
\begin{table}[ht]
\caption{Explicit for coefficients in Cor.~\ref{corollary2}, and Thm.~\ref{mainfold}.}
\label{tab:coemani}
\begin{center}
\begin{tabular}{lcc}
\multicolumn{1}{c}{\bf Coefficient}  &\multicolumn{1}{c}{\bf Explicit expressions}  
\\ \hline \\
$\widebar{A}_{t}$ &  $(1 - t^2) I_d$\\
$K_0^{*}$& $\frac{R}{1 - (1-\delta)^2}$\\
$K_1^{*}$& $3 \left( \frac{R}{1 -(1-\delta)^2} \right)^2$\\
$K_2^{*}$& $\sup_{t\leq 1-\delta}\|\frac{1}{t^2}(I-\widebar{A}_{t}^{-1})\|=\frac{1}{1-(1-\delta)^2}$\\
$K_3^{*}$& $\frac{2}{1-(1-\delta)^2}$\\
$K_4^{*}$& $\frac{2}{(1-(1-\delta)^2)^2}$\\
$K_5^{*}$& $\frac{9C_1^*R^2}{(1 -(1-\delta)^2)^2}+\frac{2}{(1-(1-\delta)^2)^2}+\frac{1}{1-(1-\delta)^2}$\\
$K_6^{*}$& $\frac{6C_2^*R^2}{(1 -(1-\delta)^2)^{2}}$\\
$K_7^{*}$& $6 (\sqrt{3}+1)  \left( \frac{R}{1-(1-\delta)^2} \right)^3 C_2^*+\frac{2R}{(1 -(1-\delta)^2)^2}$\\
$K_9^{*}$& $\frac{3\pi}{2}\,\frac{C_{2}^{*}R^{2}}{(1-(1-\delta)^{2})^{2}}
+ 6(\sqrt{3}+1)\,\frac{C_{2}^{*}R^{3}}{(1-(1-\delta)^{2})^{3}}
+ \frac{2R}{(1-(1-\delta)^{2})^{2}}.$\\
$C_1^*$&  $\max\{\frac{(1 -(1-\delta)^2)\|(1-\delta)^2(1-\frac{1}{2\|\nabla^2 h\|_{\infty}})-1\|}{\|1+(1-\delta)^2(1-\frac{1}{2\|\nabla^2 h\|_{\infty}})\|^2}, \frac{\|-(1-\delta)^2-1\|}{\|1 -(1-\delta)^2\|},\frac{\|1 -(1-\delta)^2\|}{8}\}$,\\
$C_2^*$&  $\max\{2,\frac{2(1 -(1-\delta)^2)^2(1-\frac{1}{2\|\nabla^2 h\|_{\infty}})^{\frac12}}{\|1+(1-\delta)^2(1-\frac{1}{2\|\nabla^2 h\|_{\infty}})\|^2}, \frac{9(1 -(1-\delta)^2)^2}{8\sqrt{3}(1-\delta)}$,\\
\end{tabular}
\end{center}
\end{table}
\end{proof}

\subsection{Proof of Lemma~\ref{thm2}}\label{AppenC}
\begin{proof}
First, we prove the velocity field $V(t, x)$ is well-defined at $t=0$ (\eqref{V_0} in Lemma~\ref{thm2}), i.e.
\begin{equation*}
 V(0,x) :=\lim_{t \rightarrow 0} V(t, x) = \lim_{t \rightarrow 0} \frac{x + S(t, x)}{t} = \sqrt{C}\mathbb{E}_{\widebar{p}_0}[X].
\end{equation*}
Let $ t \to 0 $, then it yields
\begin{equation*}
\lim_{t \rightarrow 0} V(t, x) = \lim_{t \rightarrow 0} \partial_t S(t, x) = \lim_{t \rightarrow 0} \left\{ \frac{C\nabla (\partial_{t} p_{t}(x))}{p_{t}(x)}-\frac{\partial_t p_{t}(x)}{p_{t}(x)}S(t,x)\right\}.
\end{equation*}
By simple calculation, it holds that
\begin{equation*}
\begin{aligned}
&\frac{\nabla(\partial_tp_{t}(x))}{p_{t}(x)}\\
=&-\partial_t \,(\text{det} B(t))(2\,\text{det} B(t))^{-1}d\left(-{K(t)^TK(t)}\big(CB(t)\big)^{-1}-{\widebar{A}_{t}^{-1}}\right)(x^Tx)x\\
&-\partial_tB(t)\big(2B(t)\big)^{-1}K(t)(\sqrt{C}B(t))^{-1}\int_{\mathbb{R}^d} yp(1,y|t,x)\mathrm{d}y\\
&+\left(-K(t)\partial_tK(t)(CB(t))^{-1}-K(t)^TK(t)\partial_t\big(B(t)^{-1}\big)\big(2CB(t)\big)^{-1}-\frac{1}{2}\partial_t(\widebar{A}_{t}^{-1})\right)\\
&\quad\cdot\Bigg(2x+\Big(-K(t)^TK(t)(CB(t))^{-1}-{\widebar{A}_{t}^{-1}}\Big)(x^Tx)x\\
&\quad\quad+2K(t)\big(\sqrt{C}B(t)\big)^{-1}\int_{\mathbb{R}^d} (x^Ty)xp(1,y|t,x)\mathrm{d}y\Bigg)
\end{aligned}
\end{equation*}
\begin{equation*}
\begin{aligned}
&+\left(\partial_tK(t)\big(\sqrt{C}B(t)\big)^{-1}+K(t)\partial_t(B(t)^{-1}\big)(\sqrt{C})^{-1}\right)\\
&\quad\cdot\Bigg(\int_{\mathbb{R}^d}\left(y+K(t)(\sqrt{C}B(t))^{-1}(x^Ty)y\right)p(1,y|t,x)\mathrm{d}y\\
&\quad\quad+\int_{\mathbb{R}^d}\Big(-{K(t)^TK(t)}(CB(t))^{-1}-\widebar{A}_{t}^{-1}\Big)(x^Ty)xp(1,y|t,x)\mathrm{d}y\Bigg)\\
&-\frac{1}{2}\partial_t\big(B(t)^{-1}
\big)\int_{\mathbb{R}^d}\Bigg(\Big(-{K(t)^TK(t)}\big(CB(t)\big)^{-1}-\widebar{A}_{t}^{-1}\Big)(y^Ty)x\Bigg)p(1,y|t,x)\mathrm{d}y\\
&-\frac{1}{2}\partial_t\big(B(t)^{-1}\big)\int_{\mathbb{R}^d}K(t)\big(\sqrt{C}B(t)\big)^{-1}(y^Ty)yp(1,y|t,x)\mathrm{d}y,
\end{aligned}
\end{equation*}
while
\begin{equation*}
\begin{aligned}
\frac{\partial_tp_{t}(x)}{p_{t}(x)}=&-\partial_t \,\big(\text{det} B(t)\big)\big(2\,\text{det} B(t)\big)^{-1}d\\
&+x^T\left(-K(t)\partial_tK(t)\big(CB(t)\big)^{-1}-(2C)^{-1}K(t)K(t)^T\partial_t\big(B(t)^{-1}\big)-\frac{1}{2}\partial_t(\widebar{A}_{t}^{-1})\right)x\\
&+\int_{\mathbb{R}^d} x^T\left(\partial_tK(t)\big(\sqrt{C}B(t)\big)^{-1}+K(t)(\sqrt{C})^{-1}\partial_t\big(B(t)^{-1}\big)\right)yp(1,y|t,x) \, \mathrm{d}y\\
&-\int_{\mathbb{R}^d} y^T\frac{1}{2}\partial_t\big(B(t)^{-1}\big)yp(1,y|t,x) \, \mathrm{d}y.
\end{aligned}
\end{equation*}
From observing that
\begin{equation*}
\lim_{t \rightarrow 0} \partial_t \,(\text{det} B(t))=0,\, \lim_{t \rightarrow 0} \partial_t B(t)=0,\,\lim_{t \rightarrow 0} K(t)=0,\, \lim_{t \rightarrow 0} \partial_t (\widebar{A}_{t}^{-1})=0,
\end{equation*}
and Assumption~\ref{assump1}, which ensures
\begin{equation*}
\lim_{t \rightarrow 0}\int_{\mathbb{R}^d}|y|^3p(1,y|t,x)\mathrm{d}y=\int_{\mathbb{R}^d}|y|^3\lim_{t \rightarrow 0}p(1,y|t,x)\mathrm{d}y=\mathbb{E}_{\widebar{p}_0}[|X|^3]<+\infty,
\end{equation*}
we have
\begin{equation*}
\lim_{t \rightarrow 0} \frac{\partial_tp_{t}(x)}{p_{t}(x)} S(t, x) = -\frac{x^\top x}{\sqrt{C}} \mathbb{E}_{\widebar{p}_0}[X], \quad 
\lim_{t \rightarrow 0} \frac{C\nabla(\partial_tp_{t}(x))}{p_{t}(x)}= \sqrt{C}\mathbb{E}_{\widebar{p}_0}[X] -\frac{x^\top x}{\sqrt{C}}  \mathbb{E}_{\widebar{p}_0}[X].
\end{equation*}
Therefore, it yields $\lim_{t \rightarrow 0} V(t, x) = \sqrt{C}\mathbb{E}_{\widebar{p}_0}[X]$, which completes the proof of \eqref{V_0}.

Next, together with the regularity of the velocity field (Theorem~\ref{thm1}), the Arzelà-Ascoli theorem~\citep{arzela1895sulle} ensures the existence of a subsequence $\{{V(t,x)}_{n_k}\}_{k\in N}$ that converges locally uniformly to ${V(0,x)}$, thereby guaranteeing the well-posedness of the ODE~\eqref{follmer} on the entire time interval $t\in[0,1]$.
\end{proof}

\subsection{Proof of Corollary~\ref{thm333}}\label{AppenC.1}
\begin{proof}
Recall the F\"{o}llmer flow \eqref{follmer} with $\|\nabla V(t,\cdot)\|_\infty\leq (K_1+K_2)t$ in Theorem~\ref{thm1}, by following the Proposition \ref{lemma_ref5}~ \citep{mikulincer2023lipschitz}, we arrive at the following result,
\begin{equation*}
\text{Lip}(\overset{\leftarrow}{X}_1(x)) \leq\|\nabla \overset{\leftarrow}{X}_1(x)\|_{op} \leq\exp \left( \int_0^1 (K_1+K_2)s ds \right) \leq\exp \left( \frac{K_1+K_2}{2}\right).
\end{equation*}
\end{proof}

\subsection{Proof of Corollary~\ref{thm3}}\label{AppenC.2}
\begin{proof}
For any $x, y \in \mathbb{R}^d$, $ t \in [t_n, t_{n+1}]$ with $k = 0, 1, \ldots, N-1$,  It\^o's  formula gives
\begin{equation*}
\begin{aligned}
\frac{\mathrm{d}|\overset{\leftarrow}{Y}_t(x)-\overset{\leftarrow}{Y}_t(y)|^2}{dt}&= 2\langle \overset{\leftarrow}{Y}_t(x)-\overset{\leftarrow}{Y}_t(y), \widetilde{V}\big(t_n,\overset{\leftarrow}{Y}_{t_n}(x)\big)-\widetilde{V}\big(t_n,\overset{\leftarrow}{Y}_{t_n}(y)\big)\rangle.
\end{aligned}
\end{equation*}
Using the Cauchy-Schwarz inequality, we obtain
\begin{equation*}
\begin{aligned}
\frac{\mathrm{d}|\overset{\leftarrow}{Y}_t(x)-\overset{\leftarrow}{Y}_t(y)|^2}{dt}&\leq 2\sqrt{|\overset{\leftarrow}{Y}_t(x)-\overset{\leftarrow}{Y}_t(y)}|^2
\sqrt{|\widetilde{V}\big(t_n,\overset{\leftarrow}{Y}_{t_n}(x)\big)-\widetilde{V}\big(t_n,\overset{\leftarrow}{Y}_{t_n}(y)\big)|^2}.
\end{aligned}
\end{equation*}
Therefore,
\begin{equation*}
\begin{aligned}
 \frac{d|\overset{\leftarrow}{Y}_t(x)-\overset{\leftarrow}{Y}_t(y)|}{dt}&\leq|\widetilde{V}\big(t_n,\overset{\leftarrow}{Y}_{t_n}(x)\big)-\widetilde{V}\big(t_n,\overset{\leftarrow}{Y}_{t_n}(y)\big)|\\
 &\leq\nabla \widetilde{V}|\overset{\leftarrow}{Y}_{t_n}(x)-\overset{\leftarrow}{Y}_{t_n}(y)|\\
 &\leq(K_1+K_2+K_8)t_n|\overset{\leftarrow}{Y}_{t_n}(x)-\overset{\leftarrow}{Y}_{t_n}(y)|,
 \end{aligned}
\end{equation*}
where the last inequality uses Lipschitzness of $\widetilde{V}$ in Assumption~\ref{assump4}.
Integration over time yields
\begin{equation*}
\begin{aligned}
   |\overset{\leftarrow}{Y}_{t_{n+1}}(x)-\overset{\leftarrow}{Y}_{t_{n+1}}(y)|
    &\leq  |\overset{\leftarrow}{Y}_{t_n}(x)-\overset{\leftarrow}{Y}_{t_n}(y)|+(t_{n+1}-t_n)(K_1+K_2+K_8)t_n|\overset{\leftarrow}{Y}_{t_n}(x)-\overset{\leftarrow}{Y}_{t_n}(y)|.
 \end{aligned}
\end{equation*}
Then, it follows that the Lipschitz constant of the discrete flow satisfies $\text{Lip}(\widetilde{T}_n)\leq1+(t_{n+1}-t_n)(K_1+K_2+K_8)t_n$. Iterating this bound over all $n = 0, 1, \ldots, N-1$, we obtain the following estimate over the full interval $[0,1]$,
\begin{equation*}
\begin{aligned}
     |\overset{\leftarrow}{Y}_{1}(x)-\overset{\leftarrow}{Y}_{1}(y)|&\leq \left(\prod_{n=0}^{N-1}\text{Lip}(\widetilde{T}_n)\right)|\overset{\leftarrow}{Y}_{t_0}(x)-\overset{\leftarrow}{Y}_{t_0}(y)|\\
    &\leq \exp\left(\frac{K_1+K_2+K_8}{2}\right)|x-y|.
\end{aligned}
\end{equation*}
Then we complete the proof.
\end{proof}

\subsection{Proof of  Theorem~\ref{thm4.2}}\label{AppenDD.1}
\begin{proof}
Recall that Assumption~\ref{assump1}, \ref{assump2} ensure the well-podeness and Lipschitzness of F\"{o}llmer flow $(\overset{\leftarrow}{X}_t)_{t_\in [0,1]}$ in \eqref{follmer}, with 
\begin{equation*}
\text{Lip}({T}_n)\leq\exp\left(\int_{t_n}^{t_{n+1}}(K_1+K_2)t dt\right) \quad\text{and}\quad\prod_{j=0}^{N-1} \text{Lip}({T}_j)\leq\exp\left( \frac{K_1+K_2}{2}\right),
\end{equation*}
as established in Lemma~\ref{thm2} and Corollary~\ref{thm333}. 

Furthermore, Assumption~\ref{assump4} guarantees the Lipschitzness of learned discret F\"{o}llmer flow  $(\overset{\leftarrow}{Y}_t)_{t_\in [0,1]}$ in \eqref{discret-follmer}, with 
\begin{equation*}
\text{Lip}(\widetilde{T}_n)\leq1+(t_{n+1}-t_n)(K_1+K_2+K_8)t_n\quad\text{and}\quad \prod_{j=0}^{N-1} \text{Lip}(\widetilde{T}_j)\leq\exp\left(\frac{K_1+K_2+K_8}{2}\right),
\end{equation*}
as shown in Corollary~\ref{thm3}. 

Therefore, it only remains to verify that the stepwise approximation error satisfies \textbf{Assumption~\ref{assump333}}. 
To analyze the discretization error at each step, we recall the expression in ~\eqref{T_n}:
\begin{equation*}
T_n(\overset{\leftarrow}{X}_{t_{n}}) = \overset{\leftarrow}{X}_{t_{n+1}}.
\end{equation*}
Applying the vector-valued Taylor expansion of $\overset{\leftarrow}{X}_{t_{n+1}}$ over $[t_n, t_{n+1}]$, the remainder  is defined by
\begin{equation*}
\begin{aligned}
R(t):=T_n(\overset{\leftarrow}{X}_{t_{n}})-\overset{\leftarrow}{X}_{t_{n}}-hV(t_n,\overset{\leftarrow}{X}_{t_{n}}).
\end{aligned}
\end{equation*}
Under Assumption~\ref{assump1}, we can derive the second moment bound by the forward diffusion process \eqref{forward}
\begin{equation*}
\begin{aligned}
\mathbb{E}_{\widebar{p}_t}|\overset{\rightarrow}{X}_{t}|^2&=\mathbb{E}_{\widebar{p}_t}|\overset{\rightarrow}{X}_{t}-(1-t)\overset{\rightarrow}{X}_{0}|^2+\mathbb{E}_{\widebar{p}_t}|(1-t)\overset{\rightarrow}{X}_{0}|^2\\
&\leq t(2-t)Tr(C)+(1-t)^2M_2\\
&\leq M_0.
\end{aligned}
\end{equation*}
Then the expectation of the $R(t)$ is controlled by 
\begin{equation*}
\begin{aligned}
\mathbb{E}|R(t)|^2&=\mathbb{E}_{\overset{\leftarrow}{X}_{t_n}\sim \overset{\leftarrow}{P}_{t_n}}|T_n(\overset{\leftarrow}{X}_{t_n})-\overset{\leftarrow}{X}_{t_n}-hV(t_n,\overset{\leftarrow}{X}_{t_n})|^2\\
&\leq\frac{h^4}{4}\sup_{\tau\in(t_n,t_{n+1})}\mathbb{E}_{\overset{\leftarrow}{X}_{\tau}\sim \overset{\leftarrow}{P}_{\tau}}|\partial_\tau V(\tau,\overset{\leftarrow}{X}_{\tau})|^2\\
&\leq \frac{3h^4}{4}\left(K_5^2{M_0}+ \frac{K_6^2}{1-t^2}+ K_7^2\right),\quad \forall t \in [0,1),
\end{aligned}
\end{equation*}
where the last inequality follows from the bound on $|\partial_tV(t,x)|$ in Theorem~\ref{thm1}, which gives
\begin{equation*}
    \mathbb{E}_{\overset{\leftarrow}{X}_{\tau}\sim \overset{\leftarrow}{P}_{\tau}}|\partial_\tau V(\tau,\overset{\leftarrow}{X}_{\tau}))|^2\leq 3\left(K_5^2{M_0}+\frac{K_6^2}{1-\tau^2}+ K_7^2\right).
\end{equation*}

Consequently, the local truncation error is bounded by
\begin{equation}\label{I2_11}
\begin{aligned}
&\sqrt{\mathbb{E}_{\overset{\leftarrow}{X}_{t_n}\sim \overset{\leftarrow}{P}_{t_n}}|T_n(\overset{\leftarrow}{X}_{t_{n}})-\widetilde{T}_n(\overset{\leftarrow}{X}_{t_{n}})|^2}\\
=&\sqrt{\mathbb{E}_{\overset{\leftarrow}{X}_{t_n}\sim \overset{\leftarrow}{P}_{t_n}}|\overset{\leftarrow}{X}_{t_n}+hV(t_n,\overset{\leftarrow}{X}_{t_n})+R(t_n)-\overset{\leftarrow}{X}_{t_n}-h\widetilde{V}(t_n,\overset{\leftarrow}{X}_{t_n})|^2 }\\
\leq&\sqrt{h^2\mathbb{E}_{\overset{\leftarrow}{X}_{t_n}\sim \overset{\leftarrow}{P}_{t_n}}|V(t_n,\overset{\leftarrow}{X}_{t_n})-\widetilde{V}(t_n,\overset{\leftarrow}{X}_{t_n})|^2}+\sqrt{\mathbb{E}|R(t_n)|^2}\\
\leq&h\left(\frac{\sqrt{3}}{2}\Big(K_5\sqrt{M_0}+ \frac{K_6}{\sqrt{1-t_n^2}}h+ K_7\Big)h+\epsilon\right).
\end{aligned}
\end{equation}
The second inequality holds by the error between $\widetilde{V}(t_n,x)$ and $V(t_n,x)$ stated in Assumption~\ref{assump3}. This completes the verification of Assumption~\ref{assump333}.

 Now in Theorem~\ref{444}, we employ coupling between $\overset{\leftarrow}{X}_0\sim \overset{\rightarrow}{P}_{1}=\gamma_C$ and $\overset{\leftarrow}{Y}_0 \sim \overset{\leftarrow}{Q}_0 = \gamma_C$, 
\begin{equation}\label{w}
\begin{aligned}  
 \mathcal{W}_2(\overset{\leftarrow}{P}_{1}, \overset{\leftarrow}{Q}_{1})
 \leq& \exp\left(\frac{K_1+K_2+K_8}{2}\right)\mathcal{W}_2(\overset{\leftarrow}{P}_{0}, \overset{\leftarrow}{Q}_{0})\\
 &+\frac{\exp\big( \frac{K_1+K_2+K_8}{2}\big)-1}{((K_1+K_2+K_8)h)/2}\cdot h \Bigg(\frac{\sqrt{3}}{2}\left(K_5\sqrt{M_0}+ K_7\right)h+\epsilon\Bigg)\\
&+\exp\left(\frac{K_1+K_2+K_8}{2}\right)\frac{\sqrt{3}K_6\pi}{4}h\\
\leq& \exp\big( \frac{K_1+K_2+K_8}{2}\big)\Bigg(\sqrt{3}\left(K_5\sqrt{M_0}+ K_9\right)h+2\epsilon\Bigg).
\end{aligned}
\end{equation}
where
a straightforward calculation shows $\displaystyle\sum_{n=0}^{N-1}\frac{K_6}{\sqrt{1-t_n^2}}=\frac{K_6\pi}{2}$; Accordingly, set $\displaystyle K_9:=\frac{K_6\pi}{4}+K_7$.
Noting that $\overset{\leftarrow}{P}_{1}=\overset{\rightarrow}{P}_{0}$, we obtain the conclusion in~\eqref{con}.
\end{proof}

\subsection{Result under the relaxation of Assumption~\ref{assump4}} \label{re}
\begin{remark}[Relaxation of Assumption~\ref{assump4}]\label{relax}
Assumption 3.13 can be relaxed as follows. We only require that for all $k \in \{0,1,\dots,N-1\}$,
\begin{equation}\label{B}
\sum_{n=0}^{N-1} (t_{n+1}-t_n) \|\nabla \widetilde{V}(t_n, \cdot) \|_\infty \le B,
\end{equation}
for some constant $B>0$. Under this condition, the Lipschitz bound stated in Corollary~\ref{thm3} is refined to $\exp(B)$. Consequently, Theorem~\ref{thm4.2} remains valid with the product of Lipschitz coefficient $\prod_{j=0}^{N-1}\text{Lip}(\widetilde{T}_j)$, where the factor $\exp\big( \frac{K_1+K_2+K_8}{2}\big)$ is replaced throughout by $\exp(B)$.
\end{remark} 
\begin{proof}
Denote $\beta_n=\|\nabla \widetilde{V}(t_n, \cdot)\|$, then
\begin{equation*}
\begin{aligned}
 \frac{d|\overset{\leftarrow}{Y}_t(x)-\overset{\leftarrow}{Y}_t(y)|}{dt}\leq|\widetilde{V}\big(t_n,\overset{\leftarrow}{Y}_{t_n}(x)\big)-\widetilde{V}\big(t_n,\overset{\leftarrow}{Y}_{t_n}(y)\big)|
 \leq\beta_n|\overset{\leftarrow}{Y}_{t_n}(x)-\overset{\leftarrow}{Y}_{t_n}(y)|.
 \end{aligned}
\end{equation*}
Integration over time yields
\begin{equation*}
\begin{aligned}
   |\overset{\leftarrow}{Y}_{t_{n+1}}(x)-\overset{\leftarrow}{Y}_{t_{n+1}}(y)|
    &\leq  |\overset{\leftarrow}{Y}_{t_n}(x)-\overset{\leftarrow}{Y}_{t_n}(y)|+(t_{n+1}-t_n)\beta_n|\overset{\leftarrow}{Y}_{t_n}(x)-\overset{\leftarrow}{Y}_{t_n}(y)|,
 \end{aligned}
\end{equation*}
yielding $\text{Lip}(\widetilde{T}_n)\leq1+(t_{n+1}-t_n)\beta_n$. Recall \eqref{B}, we obtain
\begin{equation}
\text{Lip}(\overset{\leftarrow}{Y}_1(x)) \leq\prod_{j=0}^{N-1} \text{Lip}(\widetilde{T}_j)\leq \exp\left(\sum_{n=0}^{N-1}(t_{n+1}-t_n)\beta_n\right)=\exp(B).
\end{equation}
Subsequently, we analyze the impact of relaxing Assumption~\ref{assump4} on convergence.
The preliminary estimate \eqref{I2_11} still hold. 
Under the new Lipschitz coefficient $\text{Lip}(\widetilde{T}_n)\leq1+(t_{n+1}-t_n)\beta_n$, equation \eqref{w} takes the following form,
\begin{equation}
\begin{aligned}  
 &\mathcal{W}_2(\overset{\leftarrow}{P}_{1}, \overset{\leftarrow}{Q}_{1})\\
 \leq&\exp\left(B\right)\mathcal{W}_2(\overset{\leftarrow}{P}_{0}, \overset{\leftarrow}{Q}_{0})+\frac{\exp\big(B\big)}{h}\cdot h \Bigg(\frac{\sqrt{3}}{2}\left(K_5\sqrt{M_0}+ K_7\right)h+\epsilon\Bigg)+\exp\left(B\right)\frac{\sqrt{3}K_6\pi}{4}h\\
\leq& \exp\big( B\big)\Bigg(\frac{\sqrt{3}}{2}\left(K_5\sqrt{M_0}+ K_9\right)h+\epsilon\Bigg).
\end{aligned}
\end{equation}
where $K_5$ and $K_9$ are same with Theorem~\ref{thm4.2}.
\end{proof}

\subsection{Proof of convergence theory of 1-rectified flow}\label{Interpolation}
\begin{proof}
The forward process in the first-step rectification~\citep{liu2022flow, rout2024semantic}, constructed
by independently coupling the data with a standard Gaussian reference
distribution, is defined by the interpolation $\hat{X}_t =\hat{\alpha}_t \hat{X}_1 + \hat{\beta}_t \mathcal{N}$ with $\hat{\alpha}_t = 1-t, \hat{\beta}_t = t$. 
Then the transition probability distribution from $\hat{X}_0$ to $\hat{X}_t$ is given by 
\begin{equation}
\hat{X}_t |\hat{X}_1 = x_1 \sim \mathcal{N}\big((1 - t)x_1, tI_d\big).
\end{equation}
Under the Gaussian tail Assumption~\ref{assump2}, the score function can be calculated by
\begin{equation*}
\begin{aligned}
\hat{S}(t, x):=\nabla\log \hat{p}_t 
=\nabla \log \int_{\mathbb{R}^d} \left( 2\pi\, \text{det} (\hat{B}(t)) \right)^{-\frac{d}{2}} G(t,x,y)\cdot \exp\left(-\frac{|x|_{\hat{A}_{t}}}{2}\right)\exp \left( h(y) \right)\, \mathrm{d}y,
\end{aligned}
\end{equation*}
where $G(t,x,y):=\exp\left( -\frac{|\hat{K}(t)x -y|^2_{\hat{B}(t)}}{2}\right)$, $\hat{A}_{t}=At^2+(1-t)^2I_d$, $\hat{K}(t)=(A\hat{A}_{t}^{-1})t$, $\hat{B}(t)=(A\hat{A}_{t}^{-1})(1-t)^2$.\\
First, we consider the modified score function: 
\begin{equation*}
\begin{aligned}
\widetilde{\hat{S}}(t,x)&:=\hat{S}(t,x)+\hat{A}_{t}^{-1}x\leq |\hat{K}(t)\nabla h(x)|.
\end{aligned}
\end{equation*}
Let $\hat{K}=\sup_{0\leq t\leq 1}{|\frac{1}{t}\hat{K}(t)|}=\sup_{0\leq t\leq 1}{|A \hat{A}_{t}^{-1}|}\leq1+ \|A\|$, then we have
\begin{equation*}
|\widetilde{\hat{S}}(t,\cdot)|_\infty\leq \hat{K}|\nabla h|_\infty t= \hat{K}_0t
\end{equation*}
with $\hat{K}_0:=\hat{K}|\nabla h|_\infty$.\\
Taking the derivative twice along that direction and using the same method as above, it yields
\begin{equation*}
\|\nabla\widetilde{\hat{S}}(t,\cdot)\|_\infty \leq \hat{K}(t)^2(\|\nabla ^2 h\|_\infty+|\nabla h|^2_\infty)=\hat{K}_1t^2,
\end{equation*}
where  $\hat{K}_1:=\hat{K}^2(\|\nabla ^2 h\|_\infty+|\nabla h|^2_\infty)$. 
Define $\hat{K}_2:=\sup_{0\leq t\leq1}\|\frac{1}{t}\big(I_d-(1-t)\hat{A}_{t}^{-1})\big\|=\sup_{0\leq t\leq1}\|(A+C-I_d)\big(At^2+(1-t)^2I_d\big)^{-1}\|$, we obtain 
\begin{equation*}
\begin{aligned}
    |\hat{V}(t,x)|:=\left|\frac{x+(1-t)\hat{S}(t,\cdot)}{t}\right|=\left|\frac{(1-t)\widetilde{\hat{S}}(t,\cdot)+\big(I_d-(1-t)\hat{A}_{t}^{-1}\big)x}{t}\right|\leq \hat{K}_0+\hat{K}_2|x|,
\end{aligned} 
\end{equation*}
and 
\begin{equation*}
\begin{aligned}
    \|\nabla \hat{V}(t,x)\|_\infty=\left\|\nabla\left(\frac{(1-t)\widetilde{\hat{S}}(t,\cdot)+(I_d-(1-t)\hat{A}_{t}^{-1})x}{t}\right)\right\|_\infty\leq \hat{K}_1t+\hat{K}_2.
\end{aligned}
\end{equation*}
Similar to the proof of $|\partial_t V|$ of F\"{o}llmer flow in Appendix B.2, we derive
\begin{equation*}
\begin{aligned}
|\partial_t\hat{V}(t,x)|\leq  \hat{K}_5|x|+\frac{\hat{K}_6}{\sqrt{1-t^2}}+ \hat{K}_7.
\end{aligned}
\end{equation*}
where  $\hat{K}_5,\hat{K}_6$ and $\hat{K}_7$ are dimension-free constants.
This completes the argument that the trajectory $\hat{X}_t$ of $1$-rectified flow possesses the similar regularity properties for its velocity field as those established in Theorem~\ref{thm1}. Consequently, by following the proof
steps of Theorem~\ref{thm4.2} in Appendix~\ref{AppenDD.1}, the desired result directly follows.
\end{proof}

\section{Relation to Prob ODE}\label{sec:Prob_ODE}
The probabilistic ODE (Prob ODE)~\citep{song2020score,gao2024convergence1} 
\begin{equation}\label{Prob ODE}
\frac{d \widehat{X}_s}{ds}=-(\widehat{X}_s-\nabla\log \widehat{p}_s(\widehat{X}_s)), \quad s \in[T,0],
\end{equation}
can be viewed as a time-rescaled F\"{o}llmer flow, via $s \mapsto \ln\left(\frac{1}{t}\right)$, where $T$ is finite time. Since Lipschitzness of the transport maps are invariant under time rescaling, the results of Corollary~\ref{corollary1} apply directly to the \eqref{Prob ODE}. The discretization can be chosen as
\begin{equation*}
 s_n = \ln\left(\frac{1}{t_n}\right), \quad n=1,\dots,N. 
\end{equation*}
In the forward Prob ODE, the distribution approaches Gaussian only as $s \to +\infty$. To realize this limit in practice, we set $s_0 = +\infty$ and initialize the dynamics with $\widehat{X}_{s_0} \sim \mathcal{N}(0,C)$, and take $T=s_1=\ln\left(\frac{1}{t_1}\right)$ with $\widehat{X}_{s_1} =\widehat{X}_{s_0}+t_1\sqrt{C}\mathbb{E}_{\widebar{p}_0}[X]$,  
which corresponds to a single-step first-order Euler method.
For $n$-step ($n\geq2$), the update $\displaystyle\widehat{X}_{s_n}=\widehat{X}_{s_{n-1}}+(e^{s_{n-1}-s_n}-1)\Big(S_\theta(e^{-s_{n-1}},\widehat{X}_{s_{n-1}})+\widehat{X}_{s_{n-1}}\Big)$ follows the exponential Euler scheme.
The result of Corollary~\ref{corollary1} indicates that our method improves the computational complexity of the Prob ODE, whereas \citet{wang2024wasserstein} shows $N = \mathcal{O}\left(\frac{M_0}{\epsilon_0^2}\left(\log\frac{M_2 + \mathrm{Tr}(C)}{\epsilon_0^2}\right)^3\right)$ under the same setting.

\section{Convergence in the Bayesian Inverse problems}\label{Bayesian_part}
We are aware of several posterior analyses, such as Bayesian inverse problems~\citep{van2021bayesian}, used in uncertainty quantification to infer model parameters $x$ from observations $y \in \mathbb{R}^m$. The posterior typically takes the form of
\begin{equation}\label{Bay_p0}
\widebar{p}_0(x)
= D_0 \exp\!\left(-\frac{|x|^2_C}{2}\right)
      \exp\!\left(-\frac{|G(x)-y|^2_\Sigma}{2}\right),
\end{equation}
where $D_0$ is a normalizing constant, $C$ denotes the covariance matrix of the Gaussian prior, $\Sigma$ represents the covariance of the observational noise and $G \in C^b_2(\mathbb{R}^d, \mathbb{R}^m)$ is a nonlinear forward operator.
In our training framework, we adopt the Gaussian prior with covariance 
$C$ from \eqref{Bay_p0} as the invariant measure of the forward diffusion process \eqref{forward}. The conditioned score~\citep{batzolis2021conditional} in the score matching is trained by minimizing
\begin{equation*}
\mathbb{E}_{\widebar{p}_t(x;y)} \big| \widetilde{s}(1-t,x;y) - C\nabla_x \log \widebar{p}_{t}(x;y)\big|^2,
\end{equation*}
where $\widebar{p}_{t}$ denotes the joint law of $(X_t,Y)$ with $Y=G(X_0)+\mathcal{N}(0,\Sigma)$.
For ODE-based generation of the posterior distribution with observation $y$, we impose the following assumption on the approximation error of the velocity field $V(t,x;y)$ given in \eqref{follmer}.
\begin{assumption}\label{Accuracy_Bay}
Fixing observation $y$, for each time discretization point $t_n$,
\begin{equation*}
\mathbb{E}_{\overset{\rightarrow}{P}_{1-t_n};y}  |V(t_n,x;y) - \widetilde{V}(t_n,x;y)|^2 \leq \epsilon^2.
\end{equation*}
\end{assumption}

\begin{theorem}\label{thmBayesian}
Suppose third moment Assumption~\ref{assump1}, accuracy Assumption~\ref{Accuracy_Bay} and regularity Assumption~\ref{assump4} hold. Using the Euler method to the F\"{o}llmer flow with uniform step size $h = t_{n+1} - t_n \leq 1$ yields,
\begin{equation}\label{Bay}
\mathcal{W}_2(\overset{\rightarrow}{P}_0(\cdot,y), \overset{\leftarrow}{Q}_{1}(\cdot,y))  \leq \exp\big( \frac{\widetilde{K_1}+K_8}{2}\big)\Bigg(\frac{\sqrt{3}}{N}\left(\widetilde{K_5}\sqrt{M_0}+ \widetilde{K_9}\right)+2\epsilon\Bigg).
\end{equation}
where $\widetilde{K_1},\widetilde{K_5},\widetilde{K_9}$ are dimension-free constants depending on $(\|C\|, \|\Sigma\|, G, y)$, see Table~\ref{tab:coe_bay}, and the constant $K_8$ is defined in Assumption~\ref{assump4}.
\end{theorem}
\begin{proof}
Take $A = C$, and $h(x) = -\frac{|G(x)-y|^2_{\Sigma}}{2}$, then $h(x)$ satisfies,
\begin{equation*}
\begin{aligned}
|\sqrt{C} \nabla h(x)| &= |\sqrt{C} \nabla G(x) \Sigma^{-1} (G(x) - y)| 
\le \|C\|^{\frac12} \left(|G|_\infty + |y|\right) \|\Sigma^{-1}\| \|\nabla G\|_\infty, \\
\|C \nabla^2 h(x)\| &= \|C \nabla^2 G(x) \Sigma^{-1} (G(x) - y) 
+ C \nabla G(x) \Sigma^{-1} \nabla G(x)^T \| \nonumber \\
&\le \|C\| \|\Sigma\|^{-1}  \|\nabla^2 G\|_\infty (|G|_\infty 
+ |y| \|\nabla G\|_\infty^2 ).
\end{aligned}
\end{equation*}
Then by Theorem~\ref{thm4.2}, we obtain the bound~\eqref{Bay} with the constants replaced as specified in Table~\ref{tab:coe_bay}.
\begin{table}[ht]
\caption{Explicit for coefficients in Thm.~\ref{thmBayesian}.}
\label{tab:coe_bay}
\begin{center}
\begin{tabular}{lcc}
\multicolumn{1}{c}{\bf Coefficient}  &\multicolumn{1}{c}{\bf Explicit expressions}  
\\ \hline \\
$\widetilde{K_0}$& $\| C \| \| \Sigma^{-1} \| \, \| \nabla G \|_{\infty} \, ( | G |_{\infty} + |y| )$\\
$\widetilde{K_1}$& $\| C \| \Big( \| \Sigma^{-1} \| \big( \| \nabla G \|_{\infty}^2 + ( | G |_{\infty} + |y| ) \| \nabla^2 G \|_{\infty} \big) + \| \Sigma^{-2} \| \, \| \nabla G \|_{\infty}^2 \, ( \| G \|_{\infty} + |y| )^2 \Big)$\\
$\widetilde{K_5}$& $\max\{3\widetilde{K_1},2(1+\sqrt{2})\widetilde{K_0}\widetilde{K_1}^{\frac12}\}$\\
$\widetilde{K_6}$ &$4\|C\|^{\frac12}\widetilde{K_1}$\\
$\widetilde{K_7}$ & $\max\{\widetilde{C_3}, \widetilde{C_4}\}$\\
$\widetilde{K_9}$ &$\|C\|^{\frac12}\widetilde{K_1}\pi+\widetilde{K_7}$\\
$\widetilde{C_3}$& $4\widetilde{K_1}(\widetilde{K_1}^{\frac12}+\widetilde{K_0})$\\
$\widetilde{C_4}$ &$4\sqrt{3}\widetilde{K_0}\widetilde{K_1}+\widetilde{K_0}^{\frac32}\widetilde{K_1}^{\frac12}+\frac{(1+\sqrt{6})^2}{2(1+\sqrt{2})}\widetilde{K_0}^2\widetilde{K_1}^{\frac12}$
\end{tabular}
\end{center}
\end{table}
\end{proof}

\begin{remark}
With fixed $\widetilde{K}_1, \,\widetilde{K}_5$ and $\widetilde{K}_9$, for $\epsilon_0$ accuracy in $W_2$ distance for~\ref{Bay}, one requires at most:,
\begin{equation*}
N = \mathcal{O}\left(\frac{\sqrt{M_0}}{\epsilon_0}\right), \quad  \epsilon = \mathcal{O}(\epsilon_0).
\end{equation*}
\end{remark}

\section{Conclusion and future directions}\label{conclusion}
We have established a rigorous pathway toward the optimal $\sqrt{d}$ complexity bound for flow-based generative models under the Wasserstein metric. Our approach quantifies the temporal scaling of truncation errors and controls their accumulation through dimension-free Lipschitz estimates of the backward flow. We further verify this framework in the special case of the F\"{o}llmer flow and $1$-rectified flow, where well-posedness and convergence hold under the Gaussian tail assumption. Such an assumption accommodates both regular and singular targets (with early stopping) and extends naturally to infinite-dimensional settings, with implications for Bayesian inverse problems.\\
Several directions merit further investigation:
\begin{itemize}
\item Generalization of Gaussian tail. Weak log-concavity~\cite{bruno2025wasserstein,gentiloni2025beyond} generalizes the Gaussian tail assumption, yet current results under this broader condition yield only $\mathcal{O}(d)$ complexity. Determining whether the 
$\mathcal{O}(\sqrt{d})$ bound can persist in this broader assumption remains a central theoretical challenge.
\item Higher-order integrators. Extending the error analysis to higher-order numerical schemes requires refined regularity assumptions on the velocity field, including higher-order time derivatives. Determining the minimal regularity required for higher-order Wasserstein convergence will be pursued in our subsequent work.
\item Designing of learning objectives. Alternative training objectives, temporal reweighting strategies, and adaptive sampling schemes may significantly influence both the theoretical error bounds and the practical sampling performance. Understanding these effects in a principled manner will be an important direction of our future work.
\item Step-size optimization. While the present analysis employs a uniform step size, designing optimal or adaptive step-size schedules-particularly in view of the singular behavior induced by the $\frac{1}{\sqrt{1-t^2}}$ factor-will be a central direction of our forthcoming work.
\item Data-driven Lipschitz estimation. Our Lipschitz bounds are derived as analytic upper bounds and are therefore conservative. Constructing posterior or data-driven estimators for the effective Lipschitz constants of the transport maps may yield significantly sharper Wasserstein error bounds and more realistic complexity estimates.
\end{itemize}
Addressing these questions will further clarify the structural requirements for optimal-complexity sampling and broaden the applicability of flow-based methods to high-dimensional and infinite-dimensional inference.

\end{document}